\documentclass{article}

\usepackage{iclr2027_conference, times}  
\iclrfinalcopy

\usepackage{amsmath, amssymb, amsthm, graphicx}
\usepackage{hyperref}
\usepackage{url}
\usepackage{xcolor}
\usepackage[linesnumbered,ruled,vlined]{algorithm2e}
\usepackage{float}
\usepackage{tikz}
\usetikzlibrary{arrows.meta, positioning}
\usepackage{longtable}
\usepackage{booktabs}
\usepackage{multirow}
\usepackage[utf8]{inputenc}
\usepackage{array}
\usepackage{placeins}
\usepackage{graphicx}
\usepackage{caption}

\title{OrBIT: Structure-Guided Embedding Compression\\[0.4em]
\rule{\textwidth}{0.4pt}}
 
\author{%
  Yunied Puig \\
  \texttt{puigdedios@gmail.com}
  \And
  Amit Kumar Jaiswal \\
  Jay Chaudhry Software Innovation Centre \\
  Indian Institute of Technology (BHU) \\
  Varanasi, India \\
  \texttt{ amit.chr@iitbhu.ac.in}
}

\newtheorem{theorem}{Theorem}[section]
\newtheorem{lemma}[theorem]{Lemma}
\newtheorem{proposition}[theorem]{Proposition}
\newtheorem{corollary}[theorem]{Corollary}
\newtheorem{definition}[theorem]{Definition}

\newtheorem{remark}[theorem]{Remark}

\begin{document}
\maketitle
\fancyhead{}
\renewcommand{\headrulewidth}{0pt}

\begin{abstract}
Embedding tables are among the largest components of modern language models. Most compression methods fix a coding geometry such as coordinate blocks, low-rank subspaces, or unrestricted codebooks, and optimize within it. We instead ask whether the coding geometry can itself be discovered. We introduce \emph{OrBIT}, a structure-guided embedding compression framework that learns reusable local geometry from orbit dynamics and uses it to constrain a small set of shared codewords. The global reconstruction residual then decides where the fixed coding budget is spent, while redundant overlapping charts let local errors compensate one another after gluing. Our theory shows how tight-chart geometry controls distortion, how the global residual directs sequential allocation, and how data-geometry-guided refinement improves the codec. The resulting orbit machinery is compiled away, leaving a compact decoder in which the learned structure governs what is stored, where capacity is allocated, and how local information is assembled globally. Across four LLM embedding tables, OrBIT achieves $37.9\times$ compression on GPT-2 and over $23\times$ on each 7B table relative to 16-bit storage, while delivering competitive rate-distortion performance against established quantization and low-rank baselines.
\end{abstract}

\section{Introduction}
\label{intro}
Embedding tables are a major memory cost in modern language models, but most compression methods begin by imposing a coding geometry coordinate blocks, low-dimensional subspaces, or unrestricted codebooks, and then optimize within it. We ask whether the compression geometry itself can instead be discovered from dynamics. We bring orbit ideas from linear dynamics into embedding compression: from a small anchor set, learned orbit scaffolds expose reusable local structure that constrains what can be shared, while overlapping charts coordinate these local representations globally. This leads to a local-to-global compression principle in which orbit geometry discovers structure, residual information orders the fixed stage budget and determines each stage target, and redundant synthesis assembles the result through cross-chart error repair.


Let $\mathcal H=\mathbb R^d$ be the embedding Hilbert space and let $\mathcal V=\{1,\ldots,D\}$ be a finite vocabulary. An embedding table is $E=(e_1,\ldots,e_D)\in\mathcal H^D$, with $e_n\in\mathcal H$. We equip $\mathcal H^D$ with the weighted empirical table norm $\|E\|_\mu^2:=\sum_{n=1}^D\mu_n\|e_n\|_{\mathcal H}^2$, where $\mu_n>0$ and $\sum_{n=1}^D\mu_n=1$. The weights may be uniform, yielding a pure embedding-table distortion, or frequency-weighted, yielding a task-weighted distortion. Let $\mathcal Z_1,\ldots,\mathcal Z_L$ be local chart spaces and let $R_\ell:\mathcal H\to\mathcal Z_\ell$ be bounded analysis maps. We write $\mathcal Z:=\mathcal Z_1\oplus\cdots\oplus\mathcal Z_L$, $Rh:=(R_1h,\ldots,R_Lh)$, and
\[
Z^{(\ell)}:=R_\ell E=(R_\ell e_1,\ldots,R_\ell e_D)\in\mathcal Z_\ell^D.
\]

Compression in chart $\ell$ is an identity-realization problem: for $z=R_\ell e_n$, the target is $z$ itself. We therefore retain the orbit-realization mechanism of \citet{puig2026orbit}. A small anchor subset of $Z^{(\ell)}$ generates a wrapped weighted backward shift $W_\ell$, blockwise exact interpolants $J_{\ell,k}=Y_{\ell,k}Z_{\ell,k}^{\dagger}$, an orbit dictionary $\mathcal D_\ell$, and the associated scaffold $S_\ell=\operatorname{span}\mathcal D_\ell$. OrBIT, however, is not tokenwise CCDS coding. It uses this scaffold to constrain a small collection of shared additive codewords, which are compiled into a finite decoder and amortized over the vocabulary. The orbit machinery is therefore not itself what directly saves bits. Its role is to constrain and construct the shared continuous objects that make this amortization effective.

Our central contribution: OrBIT couples anchor-generated orbit dictionaries with shared additive coding on overlapping tight charts. Its analysis separates the local error that survives synthesis, derives exact chart targets for residual-directed construction and refinement, and accounts for every object retained by the finite decoder. The experiments evaluate the resulting codec and distinguish data alignment, refinement, redundancy, and compilation effects.


\section{Related Work}
Product quantization (PQ) uses disjoint coordinate subspaces, while optimized product quantization (OPQ) learns a rotation before applying the product decomposition \citep{jegou2010product, ge2013optimized}. Additive quantization (AQ) uses unrestricted shared additive codewords; residual and stacked quantizers construct codebooks sequentially and refine them against reconstruction error \citep{babenko2014additive, martinez2014stacked}. QINCo and QINCo2 further condition codebooks on the partial reconstruction \citep{huijben2024qinco, vallaeys2025qinco2}. OrBIT combines overlapping tight charts, anchor-generated orbit dictionaries, residual-directed chart scheduling, and refinement evaluated on the compiled global decoder. Shared additive coding, residual initialization, and codebook refinement are established ingredients of this construction.

 Compositional embedding codes, hashing, subspace methods, and tensor factorizations likewise amortize shared parameters across tokens \citep{chen2018learning, tito2017hash, jaiswal2023lightweight, hrinchuk2020tensorized}. Sparse coding and dictionary learning provide another route to reusable data-dependent atoms \citep{olshausen1997sparse, aharon2006k}. OrBIT inherits wrapped shifts, exact block interpolation, and CCDS realization from \citet{puig2026orbit}, and applies them to shared compression codewords. Exact interpolation fixes the full scaffold to the anchor-target span; orbit dynamics organize the dictionary and realization within that span. The construction state is discarded after compilation.

 Redundant frame representations and multiple-description coding already exploit redundancy in quantization and reconstruction \citep{goyal1998quantized, goyal2001mdc, cvetkovic2003resilience}. OrBIT uses canonical tight-frame synthesis to couple local shared-codebook representations. Theorem \ref{thm:redundant orbit add comprssion decomposition} identifies the local error component removed by synthesis; Propositions \ref{prop:nested sequential} and \ref{prop:gluing aware refinement} express the corresponding sequential and refinement objectives. Construction state is discarded, while charts, shared codewords, and \(LT\) token indices remain. Equation (\ref{eq:compiled-storage}) accounts for the transmitted chart description, codewords, and token indices.

\section{ORBIT-STRUCTURED ADDITIVE COMPRESSION AND FRAME GLUING}
\label{sec:redundant orbit based chart}
We first state the finite decoder and its storage law, then explain how the orbit framework constructs its shared codewords. We finally derive the global error decomposition, nested codebook allocation, and post-gluing refinement. Proofs are deferred to Appendix~\ref{app:proofs}.

\paragraph{Decoder and compression mechanism.}
A gluing map is a bounded linear operator $\Gamma:\mathcal Z\to\mathcal H$. It induces the pointwise table-level map $\Gamma^D:\mathcal Z_1^D\oplus\cdots\oplus\mathcal Z_L^D\to\mathcal H^D$, and we set $R_DE:=(Z^{(1)},\ldots,Z^{(L)})$. Its exact gluing defect is $\varepsilon_{\mathrm{glue}}(E):=\|E-\Gamma^D R_DE\|_\mu$.

Fix $T$ additive stages. For chart $\ell$ and stage $t$, let $\mathcal C_{\ell,t}=\{c_{\ell,t,r}\in\mathcal Z_\ell:1\le r\le Q_{\ell,t}\}$ be a shared codebook. Token $n$ stores one index $i_{\ell,n,t}\in\{1,\ldots,Q_{\ell,t}\}$ for each chart-stage pair and has decoded local representation
\begin{equation}
\widehat z_{\ell,n}:=\sum_{t=1}^T c_{\ell,t,i_{\ell,n,t}},
\qquad
\widehat Z^{(\ell)}:=(\widehat z_{\ell,1},\ldots,\widehat z_{\ell,D}).
\label{eq:orbit-additive-local}
\end{equation}
Under the canonical gluing map defined below, the final reconstructed embedding table is
\begin{equation}
\widehat E:=\Gamma^D\big(\widehat Z^{(1)},\ldots,\widehat Z^{(L)}\big).
\label{eq:orbit-additive-global}
\end{equation}
That is the entire decoder: the $\sum_{\ell,t}Q_{\ell,t}$ continuous codewords are shared, while each token stores only the $LT$ discrete indices. The remaining question is how orbit geometry constructs and constrains these shared codewords.

\paragraph{Orbit scaffold and shared-codeword realization.}
For each $1\le\ell\le L$, let $\mathcal Z_\ell:=\mathbb R^{M_\ell}$, let $\{u_{\ell,0},\ldots,u_{\ell,M_\ell-1}\}$ denote its canonical basis, and let $w_\ell=(w_{\ell,0},\ldots,w_{\ell,M_\ell-1})$ be nonzero weights. The wrapped construction of \citet{puig2026orbit} takes indices modulo the ambient dimension, so coordinates circulate rather than disappear.

\begin{definition}[The wrapped weighted backward shift, \citep{puig2026orbit}]
A wrapped weighted backward shift $W_\ell:\mathcal Z_\ell\to \mathcal Z_\ell$ is defined by
$W_\ell u_{\ell,j}=w_{\ell,j}u_{\ell,j-1\;(\mathrm{mod}\;M_\ell)}$,
for $j=0,\ldots,M_\ell-1$, where $w_{\ell,j}\neq0$.
\end{definition}

Compression in chart $\ell$ is the identity-realization problem $x_j^{(\ell)}=y_j^{(\ell)}=R_\ell e_j$. Choose a small anchor subset \(I_\ell \subseteq \{1,\ldots,D\}\), an initial partition
\(I_\ell = I_{\ell,1}\cup\cdots\cup I_{\ell,K_\ell}\), and learn \(W_\ell\).
The partition is refined as necessary until, in every block, depths
\(L_{\ell,j}\in\{0,\ldots,M_\ell-1\}\) can be chosen so that the orbit representatives
\(z_j^{(\ell)}:=W_\ell^{L_{\ell,j}}R_\ell e_j\), \(j\in I_{\ell,k}\), are linearly independent. Writing these representatives and their targets as the columns of $Z_{\ell,k}$ and $Y_{\ell,k}$, respectively, set $J_{\ell,k}:=Y_{\ell,k}Z_{\ell,k}^{\dagger}$. Then $J_{\ell,k}W_\ell^{L_{\ell,j}}R_\ell e_j=R_\ell e_j$ for every $j\in I_{\ell,k}$.

The resulting local orbit dictionary is
\begin{equation}
\label{eq:dictionary}
    \mathcal D_\ell
    :=
    \left\{
        J_{\ell,k}W_\ell^{m}R_\ell e_j
        :
        j\in I_{\ell,k},\
        1\le k\le K_\ell,\
        0\le m\le M_\ell-1
    \right\},
\end{equation}
and its scaffold is $S_\ell:=\operatorname{span}\mathcal D_\ell\subseteq\mathcal Z_\ell$, with orthogonal projector $\Pi_{S_\ell}$. 
The candidate scaffold is retained only if it captures the full local table:
\begin{equation}
\label{eq:scaffold-accept}
\varepsilon_{\rm proj,\ell}:=\|(I-\Pi_{S_\ell})Z^{(\ell)}\|_\mu\le\tau_{\rm proj}.
\end{equation}
The anchors therefore generate the orbit geometry, whereas acceptance on $Z^{(\ell)}$ ensures that this geometry controls the full-table projection bottleneck.

For a target $p\in\mathcal Z_\ell$, write $\operatorname{CCDS}_{\mathcal D_\ell}(p)\in S_\ell$ for the output of the constructive control-deviation solver of \citet{puig2026orbit}. It is a sparse linear combination of atoms of $\mathcal D_\ell$ constructed to approximate $\Pi_{S_\ell} p$, i.e. the component of $p$ captured by $S_\ell$. CCDS is our default realizer because it supplies the deviation-controlled sparse realization developed in \citep{puig2026orbit}; other sparse realizers, such as Orthogonal Matching Pursuit (OMP), can be substituted without changing the local-to-global compression architecture. Further details of the finite-dimensional orbit construction and its implementation in the present compression encoder are provided in Appendix~\ref{app:implementation} for completeness and reproducibility.
Given a prototype $v_{\ell,t,r}\in\mathcal Z_\ell$, its orbit realization is $$\widetilde c_{\ell,t,r}:=\operatorname{CCDS}_{\mathcal D_\ell}(v_{\ell,t,r}).$$
Each orbit realization $\widetilde c_{\ell,t,r}\in S_\ell\subseteq\mathbb R^{M_\ell}$ is transmitted directly in chart coordinates. If $Q_{b_\gamma}$ is the finite-bit coordinate quantizer, the orbit realization is compiled as
\begin{equation}
\label{eq:compilation}
\widetilde c_{\ell,t,r}
\longmapsto c_{\ell,t,r}:=Q_{b_\gamma}(\widetilde c_{\ell,t,r}).
\end{equation}
Every codeword created during sequential allocation or refinement (processes addressed below) passes through this same map.

\paragraph{Redundant frame gluing.}
To quantify how redundant chart errors propagate under gluing, we introduce the frame-compatible structure governing the global analysis map.
\begin{definition}[Frame-compatible chart system]
We say that the analysis maps $\{R_\ell\}_{\ell=1}^L$ form a frame-compatible chart system
with bounds $0<A\leq B<\infty$ if
\[
    A\|h\|_{\mathcal H}^2
    \leq
    \sum_{\ell=1}^L \|R_\ell h\|_{\mathcal Z_\ell}^2
    \leq
    B\|h\|_{\mathcal H}^2,
    \
    \forall h\in\mathcal H.
\]
If $A=B$, we call the system $A$-tight.
\end{definition}

In the $A$-tight case, $R^*R=A\,\mathrm{Id}_{\mathcal H}$ and $\|R\|=\sqrt A$. The canonical gluing map is $\Gamma_A:=A^{-1}R^*$, so $\Gamma_AR=\mathrm{Id}_{\mathcal H}$, $\varepsilon_{\mathrm{glue}}(E)=0$, and $\|\Gamma_A\|=A^{-1/2}$. We refer to \citet{christensen2008frames} for frame theory and to Appendix~\ref{app:normalization redundancy} for the unit-gain normalization under which $A$ measures genuine chart redundancy.

\paragraph{Local-to-global error decomposition.}
Let $\widetilde Z^{(\ell)}$ denote the analogous local additive table formed from the uncompiled orbit realizations $\widetilde c_{\ell,t,r}$, and set $\widetilde Z:=(\widetilde Z^{(1)},\ldots,\widetilde Z^{(L)})$. Define
$\varepsilon_{\rm proj}^2(E):=\sum_{\ell=1}^L\|Z^{(\ell)}-\Pi_{S_\ell}Z^{(\ell)}\|_\mu^2$,
$\varepsilon_{\rm OA}^2(E):=\sum_{\ell=1}^L\|\Pi_{S_\ell}Z^{(\ell)}-\widetilde Z^{(\ell)}\|_\mu^2$,
$\varepsilon_{\rm comp}^2(E):=\sum_{\ell=1}^L\|\widetilde Z^{(\ell)}-\widehat Z^{(\ell)}\|_\mu^2$, and
$\varepsilon_{\rm loc}^2(E):=\sum_{\ell=1}^L\|Z^{(\ell)}-\widehat Z^{(\ell)}\|_\mu^2$.
Since $\widetilde Z^{(\ell)}\in S_\ell^D$, orthogonality gives
$\|R_DE-\widetilde Z\|_\mu^2=\varepsilon_{\rm proj}^2(E)+\varepsilon_{\rm OA}^2(E)$,
and hence
$\varepsilon_{\rm loc}(E)\le(\varepsilon_{\rm proj}^2(E)+\varepsilon_{\rm OA}^2(E))^{1/2}+\varepsilon_{\rm comp}(E)$.

\begin{theorem}[Redundant orbit-additive compression decomposition]
\label{thm:redundant orbit add comprssion decomposition}
Let $E\in\mathcal H^D$, let $R:\mathcal H\to\mathcal Z_1\oplus\cdots\oplus\mathcal Z_L$ be a bounded analysis map, and let $\Gamma:\mathcal Z_1\oplus\cdots\oplus\mathcal Z_L\to\mathcal H$ be a bounded gluing map. Set $\widehat Z=(\widehat Z^{(1)},\ldots,\widehat Z^{(L)})$ and $\widehat E=\Gamma_D\widehat Z$. If
$\varepsilon_{\ker}(E):=\|\Pi_{\ker\Gamma_D}(R_DE-\widehat Z)\|_\mu$,
then
\begin{equation}
\|E-\widehat E\|_\mu
\le
\varepsilon_{\rm glue}(E)
+\|\Gamma\|
\big(\varepsilon_{\rm loc}^2(E)-\varepsilon_{\ker}^2(E)\big)^{1/2}.
\label{eq:general-orbit-additive}
\end{equation}
If $\{R_\ell\}_{\ell=1}^L$ is $A$-tight and $\Gamma=\Gamma_A=A^{-1}R^*$, then $\varepsilon_{\rm glue}(E)=0$ and
\begin{equation}
\|E-\widehat E\|_\mu
=
\frac{1}{\sqrt A}
\big(\varepsilon_{\rm loc}^2(E)-\varepsilon_{\rm inc}^2(E)\big)^{1/2},
\label{eq:tight-orbit-additive}
\end{equation}
where $\varepsilon_{\rm inc}(E):=\operatorname{dist}_\mu(\widehat Z,\operatorname{Ran}R_D)$.
\end{theorem}

The theorem separates local approximation from global frame geometry: the orbit construction controls $\varepsilon_{\rm proj}$ and $\varepsilon_{\rm OA}$, compilation contributes $\varepsilon_{\rm comp}$, while canonical synthesis removes the inconsistency component $\varepsilon_{\rm inc}$. Consequently, the relevant construction objective is the error that survives global gluing, motivating the residual-driven codebook construction below.

\begin{remark}[Relation to product-code geometry]
\label{rem:relation to product code compression}
The nonredundant orthogonal geometry underlying product codes is recovered as a special case. If
$\mathcal H=\mathcal H_1\oplus\cdots\oplus \mathcal H_F$ and $P_f: \mathcal H\to \mathcal H_f$ are the coordinate projections, then the analysis map
$R_\oplus h=(P_1h,\ldots,P_Fh)$ is an isometry and the corresponding gluing map is simple concatenation. Hence $A=1$ and there is no nontrivial inconsistency component to remove, so Theorem~\ref{thm:redundant orbit add comprssion decomposition} reduces to
\[
\|E-\widehat E_{\rm prod}\|_\mu
=
\varepsilon_{\rm loc,prod}(E).
\]
For redundant $A$-tight charts, instead, (\ref{eq:tight-orbit-additive}) holds.
Thus the distinction is not merely the factor $A^{-1/2}$: overlapping charts introduce a consistency geometry absent from product codes, and canonical gluing removes the component of the local reconstruction error lying outside $\operatorname{Ran}R_D$.
\end{remark}

\paragraph{Global-residual chart scheduling and initialization.}
Residual initialization and codebook refinement are standard in additive quantization (see Related Work). Here initialization is coupled to redundant synthesis: the compiled global residual selects the next chart and determines its target \(U^{(s)}=AR_{\ell_s}G^{(s-1)}\), which is realized through the orbit dictionary.

Set $G^{(0)}:=E$. At step $s$, let $\widehat E^{(s-1)}$ denote the global reconstruction produced by the first $s-1$ allocated stages, so that the current global residual is
$G^{(s-1)}=E-\widehat E^{(s-1)}$. Let $N_\ell^{(s-1)}$ denote the number of stages already allocated to chart $\ell$ after $s-1$ global construction steps. For $s=1,\ldots,LT$, define the eligible set
$\mathcal E_s:=\{\ell: N_\ell^{(s-1)}<T\}$ and choose
\begin{equation}
\ell_s\in\operatorname*{arg\,max}_{\ell\in\mathcal E_s}
\|R_\ell G^{(s-1)}\|_\mu,
\qquad
t_s:=N_{\ell_s}^{(s-1)}+1.
\label{eq:nested-chart-selection}
\end{equation}
The rule adapts construction order; the final budget remains \(T\) stages per chart.
Thus the next shared stage is assigned to the chart that currently sees the largest component of the unreconstructed global table. Its local target is
\begin{equation}
U^{(s)}=(u_1^{(s)},\dots,u_D^{(s)}):=A R_{\ell_s}G^{(s-1)}\in \mathcal Z_{\ell_s}^D.
\label{eq:nested-stage-target}
\end{equation}
Thus $U^{(s)}$ is not merely a local residual target: it is the component of the current globally glued reconstruction error visible to chart $\ell_s$.
Using residual prototypes from $U^{(s)}$, we construct the stage codebook $\mathcal C_{\ell_s,t_s}$ through the same orbit-realization and compilation map of (\ref{eq:compilation}). Once $C_{\ell_s,t_s}$ has been constructed, every token must choose one of its $Q$ codewords that best approximates its row of the current target $U^{(s)}$, namely
\begin{equation}
\label{eq:guaranteeing U and V are close enough}
i_{\ell_s,n,t_s}
\in
\arg\min_{1\le r\le Q}
\left\|
u_n^{(s)}
-
c_{\ell_s,t_s,r}
\right\|_2.
\end{equation}
Because the orbit realization and coordinate quantization may have moved the original prototypes $v_r^{(s)}$, this assignment is made against the actual resulting codewords $c_{\ell_s,t_s,r}$, not merely against the initial prototypes.
After these $D$ assignments have been made, we obtain
\begin{equation}
V^{(s)}
:=
\big(
c_{\ell_s,t_s,i_{\ell_s,1,t_s}},
\ldots,
c_{\ell_s,t_s,i_{\ell_s,D,t_s}}
\big)
\in \mathcal Z_{\ell_s}^D
\label{eq:nested-stage-table}
\end{equation}
which denotes the actual compiled stage table selected by the token assignments.
Its global contribution after gluing is
$
\frac1A R_{\ell_s}^*V^{(s)}.
$
Therefore the global residual evolves as
\begin{equation}
G^{(s)}
:=
G^{(s-1)}
-
\frac{1}{A}R_{\ell_s}^*V^{(s)}.
\label{eq:nested-residual-update}
\end{equation}
Hence the next stage is fitted after accounting for the actual compiled contribution of every previously allocated stage, including stages belonging to other charts.
So, after $s$ construction steps, the provisional global reconstruction is
$
\widehat E^{(s)}
=
\frac1A
\sum_{j=1}^{s}
R_{\ell_j}^*V^{(j)},
 \text{ and } 
G^{(s)}=E-\widehat E^{(s)}.
$
After all $LT$ stages, iterating \eqref{eq:nested-residual-update} from $G^{(0)}=E$ gives
\[
\widehat E
=
E-G^{(LT)}
=
\frac1A
\sum_{s=1}^{LT}
R_{\ell_s}^*V^{(s)}
=
\frac1A
\sum_{\ell=1}^L
R_\ell^*
\left(
\sum_{t=1}^T
V_{\ell,t}
\right).
\]
where the last equality regroups the $LT$ steps by chart and stage, with each $(\ell,t)$ appearing exactly once.
For token $n$,
$
\sum_{t=1}^T V_{\ell,t,n}
=
\sum_{t=1}^T
c_{\ell,t,i_{\ell,n,t}}
=
\widehat z_{\ell,n}.
$
Regrouping the sum by chart gives
\begin{equation}
\widehat E=
\frac{1}{A}
\sum_{\ell=1}^L
R_\ell^*\widehat Z^{(\ell)}
=
\Gamma_A^D
\big(
\widehat Z^{(1)},\ldots,\widehat Z^{(L)}
\big).
\end{equation}

So the recursive construction and the final decoder \eqref{eq:orbit-additive-global} are exactly the same object viewed at two different levels:
at construction level we have $V^{(1)},\dots,V^{(LT)}$ while at decoder level we have $C_{\ell,t,\ i_{\ell,n,t}}$.
We now show why $U^{(s)}$ is precisely the local target that a newly allocated stage should approximate in order to reduce the current global reconstruction residual.

\begin{proposition}[Nested sequential shared-codebook residual identity]
\label{prop:nested sequential}
Assume the unit-gain normalization $R_\ell R_\ell^*=I_{\mathcal Z_\ell}$ and canonical $A$-tight gluing. Then at every sequential step $s$,
\begin{equation}
\|G^{(s)}\|_\mu^2
=
\|G^{(s-1)}\|_\mu^2
-
\frac{1}{A^2}\|U^{(s)}\|_\mu^2
+
\frac{1}{A^2}\|U^{(s)}-V^{(s)}\|_\mu^2.
\label{eq:nested-residual-identity}
\end{equation}
 \end{proposition}

\begin{remark}
Proposition~\ref{prop:nested sequential} turns the global objective of Theorem~\ref{thm:redundant orbit add comprssion decomposition} into a sequential construction rule: once a chart has been selected, the role of the new stage is to make its realized contribution $V^{(s)}$ as close as possible to $U^{(s)}$: the better this approximation, the more of the currently visible global residual is removed. This justifies the realization criterion \eqref{eq:guaranteeing U and V are close enough}, while \eqref{eq:nested-chart-selection} prioritizes the eligible chart with the largest visible residual energy, thereby maximizing the favorable term \(\|U^{(s)}\|_\mu^2/A^2\) subtracted in \eqref{eq:nested-residual-identity}.
\end{remark}

\paragraph{Compiled decoder and storage.}
At this point the storage mechanism is completely visible.
The encoder uses $W_\ell$, $Y_{\ell,k}$, $Z_{\ell,k}$, $J_{\ell,k}$, $\mathcal D_\ell$, and CCDS to construct the shared codewords. These are encoder-side objects: after compilation, the decoder stores only the shared representation map, the chart system, the quantized coordinates of the shared codewords, the scales, the codec metadata, and the token indices.

Let $C_{\rm rep}$ denote the storage of the shared representation map, $C_{\rm chart}$ the chart system, and $C_{\rm meta}$ the remaining fixed codec metadata. If codeword coordinates use $b_\gamma$ bits, each shared codeword carries $b_s$ scale bits, and stage $t$ in chart $\ell$ contains $Q_{\ell,t}$ codewords, then
\begin{equation}
C_{\rm OrBIT}
=C_{\rm rep}+C_{\rm chart}+C_{\rm meta}
+\sum_{\ell=1}^L\sum_{t=1}^TQ_{\ell,t}(M_\ell b_\gamma+b_s)
+D\sum_{\ell=1}^L\sum_{t=1}^T\lceil\log_2Q_{\ell,t}\rceil.
\label{eq:compiled-storage}
\end{equation}
The final term is the entire token-specific cost. The shared codeword cost is $\sum_{\ell,t}Q_{\ell,t}(M_\ell b_\gamma+b_s)$, while the only term scaling with $D$ consists of the $LT$ small discrete indices per token. Thus compression arises from amortizing a small collection of orbit-structured continuous codewords over the vocabulary, not from storing a continuous sparse realization for every token.


\paragraph{Scaffold-aligned post-gluing refinement.}
After initialization, we revisit codewords while holding the remaining decoder state and assignments fixed, as in additive-codebook refinement, but with a different target: OrBIT refines each codeword against the residual of the current globally glued reconstruction while constraining every correction to the learned orbit geometry. Thus the completed reconstruction determines the exact chart target, while the learned orbit geometry controls how accurately that target can be realized by the finite decoder. Proposition~\ref{prop:gluing aware refinement} formalizes this structure-guided refinement step.

For the remainder of this section: we assume the unit-gain normalization \(R_\ell R_\ell^*=I_{Z_\ell}\), canonical \(A\)-tight gluing. Additionally, we fix a chart \(\ell\), stage \(t\), and let $c_{\rm old}:=c_{\ell,t,r_0}\in \mathcal Z_\ell$ be one of the currently stored compiled codewords in \( C_{\ell,t}\subset Z_\ell\). Define
$
I=\{n:i_{\ell,n,t}=r_0\}
$
and assume
$
m_I:=\sum_{n\in I}\mu_n>0.
$
Let \(\widehat E_{\mathrm{old}}\) denote the current compiled global reconstruction, and \(c^\star\in \mathcal Z_\ell\) be the unique unconstrained codeword minimizing the conditional post-gluing distortion when all remaining decoder objects and assignments are held fixed. 
With
$
\bar r_I=\frac1{m_I}\sum_{n\in I}\mu_n(e_n-\widehat e_{\rm old,n}),
$
the target is simply
$
c^*=c_{\rm old}+AR_\ell\bar r_I.
$
See Equation (\ref{eq:c start eq}) below for a proof of this.

Let
$
\widetilde c
:=
\operatorname{CCDS}_{D_\ell}
\!\left(c^\star\right)
$
be its orbit realization and let
$
c_{\mathrm{new}}
:=
 Q_{b_\gamma}
\!\left(\widetilde c\right)
$
be the corresponding compiled codeword. Define
$$
\delta_{\mathrm{scaf}}
:=
\|(I-\Pi_{S_\ell})c^\star\|,
\qquad
\delta_{\mathrm{orb}}
:=
\|\Pi_{S_\ell}c^\star-\widetilde c\|,
\qquad
\delta_{\mathrm{comp}}
:=
\|\widetilde c-c_{\mathrm{new}}\|,
$$
and set
$
\Delta
:=
\sqrt{\delta_{\rm scaf}^2+\delta_{\rm orb}^2}+\delta_{\rm comp}.
$
The next result shows that the decoder geometry gives an exact descent law, while the orbit scaffold gives a sufficient structural certificate for satisfying it.

\begin{proposition} (Scaffold-aligned post-gluing refinement).
\label{prop:gluing aware refinement}
Let \(\widehat E_{\mathrm{new}}\) denote the resulting reconstruction. Then

(i) (Exact finite-decoder statement)
$$
\left\|
E-\widehat E_{\mathrm{new}}
\right\|_\mu^2
-
\left\|
E-\widehat E_{\mathrm{old}}
\right\|_\mu^2
=
\frac{m_I}{A^2}
\left(
\left\|
c_{\mathrm{new}}-c^\star
\right\|^2
-
\left\|
c_{\mathrm{old}}-c^\star
\right\|^2
\right).
$$

(ii) (OrBIT-specific structural certificate) If
$
\Delta
<
\|c_{\mathrm{old}}-c^\star\|,
$
then replacing \(c_{\mathrm{old}}\) by \(c_{\mathrm{new}}\), with the remaining decoder state and assignments fixed, strictly decreases the global reconstruction error.
$$
\|E-\widehat E_{\mathrm{new}}\|_\mu^2
\le
\|E-\widehat E_{\mathrm{old}}\|_\mu^2
-
\frac{m_I}{A^2}
\left(
\|c_{\mathrm{old}}-c^\star\|^2-\Delta^2
\right).
$$
\end{proposition}

\begin{remark}
Proposition \ref{prop:gluing aware refinement} concludes that a post-gluing correction is guaranteed to improve the codec whenever the correction demanded by the original table is captured sufficiently well by the learned scaffold, orbit realization, and finite-bit compilation.
\end{remark}


The next result links refinement directly to the learned orbit geometry: the refinement scaffold error is exactly the portion of the current compiled state and chart-visible global residual that lies outside the scaffold $S_\ell$. When $c_{\rm old}\in S_\ell$, only the residual contribution remains.
\begin{lemma}[Refinement scaffold error as unexplained global residual]
\label{lem:refinement-scaffold-error}
\[
\delta_{\rm scaf}
=
\left\|
(I-\Pi_{S_\ell})c_{\rm old}
+
\frac{A}{m_I}
\sum_{n\in I}
\mu_n
(I-\Pi_{S_\ell})
R_\ell
\bigl(e_n-\widehat e_{\rm old,n}\bigr)
\right\|.
\]
In particular, if \(c_{\rm old}\in S_\ell\), then
\[
\delta_{\mathrm{scaf}}
=
\frac{A}{m_I}
\left\|
(I-\Pi_{S_\ell})
\sum_{n\in I}
\mu_n
R_\ell
\bigl(
e_n-\widehat e_{\mathrm{old},n}
\bigr)
\right\|,
\]
where $
\widehat E_{\mathrm{old}}
=
(\widehat e_{\mathrm{old},n})_{n=1}^D
$
is the current compiled global reconstruction.
\end{lemma}

Under cross-chart scaffold compatibility, the aggregate stagewise refinement scaffold error is controlled by the original table's scaffold projection error $\varepsilon_{\mathrm{proj},\ell}$ from~(\ref{eq:scaffold-accept}).
\begin{corollary}[Collapse under cross-chart scaffold compatibility]
\label{cor:refinement-scaffold-compatibility}
Under the standing unit-gain and canonical \(A\)-tight synthesis assumptions, fix a stage \((\ell,t)\) and let
$
I_r:=\{n:i_{\ell,n,t}=r\},
$ with 
$
m_{I_r}:=\sum_{n\in I_r}\mu_n>0
$
denote its nonempty assignment cells. Assume
$
c_{\ell,t,r}\in S_\ell
\quad\text{for every }r\text{ with }m_{I_r}>0,
R_\ell R_j^\ast S_j\subseteq S_\ell
\quad\text{for every }1\le j\le L,
$
and
$
\widehat z_{j,n}\in S_j
\quad\text{for every }1\le j\le L,\;1\le n\le D.
$
Then
$
\sum_{r:\,m_{I_r}>0}
\frac{m_{I_r}}{A^2}
\delta_{{\rm scaf},r}^2
\le
\varepsilon_{{\rm proj},\ell}^2.
$
\end{corollary}


\section{EXPERIMENTS}
\label{sec:experiments}

We design the experiments around five questions that correspond to the main components of OrBIT:
(i) whether the complete compiled codec gives a competitive rate--distortion tradeoff, and in particular whether imposing learned orbit structure remains advantageous relative to unconstrained additive coding; (ii) whether the learned orbit scaffold itself captures the embedding geometry better than an unstructured subspace of identical rank; (iii) whether post-gluing refinement improves the actual global reconstruction at fixed rate; (iv) whether redundancy produces the error-cancellation mechanism predicted by the theory; and (v) whether the orbit construction can genuinely be discarded after compilation. Experiments 1--5 address these questions below.

\paragraph{Experimental protocol.} 
\label{sec:experimental-protocol}
We evaluate GPT-2, Llama-2-7B, Mistral-7B-v0.1, and Mistral-7B-Instruct-v0.3. For each model we extract $12{,}000$ token embeddings and compress a fixed $D=10{,}000$-token table. GPT-2 has original dimension $768$ and is mapped to $d=64$; the three $7$B models have original dimension $4096$ and are mapped to $d=192$. The corresponding PCA explained-variance fractions are $0.2624$, $0.2200$, $0.2133$, and $0.2143$. These fractions are computed relative to the centered table. Importantly, the PCA map is not free: its complete storage cost is included in every reported rate. The uncompressed reference uses $16$ bits per coordinate, giving $12{,}288$ bits/token for GPT-2 and $65{,}536$ bits/token for the $7$B tables.

Our primary distortion measure is the relative reconstruction error
$
\frac{\|E-\widehat E\|_\mu}{\|E\|_\mu}
$
in the \emph{original embedding space}, with uniform $\mu$. Thus the reported end-to-end error includes both the fixed representation truncation and the subsequent OrBIT coding error. We additionally monitor cosine similarity and nearest-neighbor recall, but use relative error for all rate--distortion comparisons.

\paragraph{OrBIT configuration and storage.}
The standard experiments use $L=4$ unit-gain $A=2$ tight charts, with $M_\ell=32$ for GPT-2 and $M_\ell=96$ for the $7$B tables. The orbit scaffolds are constructed from $0.75\%$ of the GPT-2 table and $1.2\%$ of the $7$B tables; scaffold acceptance is always evaluated on the complete local table with $\tau_{\rm proj}=10^{-3}$. Unless stated otherwise, we use $T=7$ additive stages and $Q_{\ell,t}=32$ shared codewords per chart-stage pair. Shared codeword coordinates are quantized to $8$ bits with one $16$-bit scale per codeword, so every token-specific chart-stage assignment costs only $5$ bits. All rates include $C_{\rm rep}$, the chart-system description, every shared quantized codeword and scale, codec metadata, and all token indices. Full implementation details are given in Appendix~\ref{app:implementation}.

We compare with PQ, OPQ, AQ, scalar quantization (SQ), and truncated SVD. In addition to their standard storage conventions, we re-price PQ, OPQ, and AQ using the same $8$-bit-coordinate plus $16$-bit-scale convention as OrBIT; we denote these controls by PQ-q8, OPQ-q8, and AQ-q8. This prevents an apparent gain from being explained merely by cheaper storage of OrBIT's shared codewords.

\paragraph{Rate--distortion across four embedding tables.}
Varying \(T=1,\ldots,7\) gives a monotone OrBIT rate--distortion curve on all four tables. At \(T=7\), OrBIT compresses GPT-2 by \(37.9\times\) and each 7B embedding table by \(23.9\times\). Table~1 compares this endpoint with representative high-rate q8 baselines; the complete curves, including the standard-precision baselines, are reported in the accompanying rate--distortion Figure~\ref{fig:rate_distortion} in Appendix \ref{app:experimental-details}.

\begin{table*}[t]
\centering
\small
\setlength{\tabcolsep}{5pt}
\caption{Representative end-to-end rate--distortion results. Each entry is bits/token\,/\,relative error in the original embedding space. The PQ-q8 and OPQ-q8 entries use $m=16,b=8$; AQ-q8 uses $m=8,b=8$. All shared objects, including the PCA representation map, are charged to representative operating points with matched precision. Lower error and lower rate are better.}
\label{tab:main-rd}
\begin{tabular}{lcccc}
\hline
Embedding table & OrBIT, $T=7$ & PQ-q8 & OPQ-q8 & AQ-q8 \\
\hline
GPT-2 &
$324.15/0.7449$ &
$307.40/0.7467$ &
$313.96/0.7465$ &
$331.88/0.7582$ \\
Llama-2-7B &
$2739.97/0.9163$ &
$2703.56/0.9211$ &
$2762.55/0.9192$ &
$2911.54/0.9278$ \\
Mistral-7B &
$2739.97/0.9151$ &
$2703.56/0.9189$ &
$2762.55/0.9170$ &
$2911.54/0.9228$ \\
Mistral-7B-Instruct &
$2739.97/0.9205$ &
$2703.56/0.9176$ &
$2762.55/0.9157$ &
$2911.54/0.9205$ \\
\hline
\end{tabular}
\end{table*}

AQ is the closest comparator in decoder form, whereas PQ and OPQ test product-code representations; their structural differences are discussed in Related Work. All methods share the input representation and storage accounting, with matched shared-codeword precision in the q8 controls. At the displayed operating points, Table \ref{tab:main-rd} shows that OrBIT improves both rate and distortion over AQ-q8 on GPT-2, Llama-2-7B, and Mistral-7B, and approximately matches its distortion at lower rate on Mistral-7B-Instruct. This suggests the orbit constraint therefore does not merely reduce the freedom of the additive codebooks: across these tables it yields a more rate-efficient structured additive representation than the unrestricted AQ comparator. On the other side, OrBIT has lower distortion than the displayed PQ-q8 and OPQ-q8 points on the first three tables; both product quantizers attain lower distortion on Mistral-7B-Instruct.
 This second comparison is particularly demanding because product quantization builds strong compression efficiency directly into the factorization itself, so it asks whether the learned overlapping orbit geometry can compete with an efficient but prescribed nonoverlapping product geometry. The empirical claim is therefore not universal dominance, but that shared codewords constrained by learned orbit geometry can outperform unconstrained additive coding and remain competitive with mature product-quantization baselines while providing an explicitly analyzable local-to-global compression mechanism. The SVD baseline is substantially less rate-efficient in this regime, while SQ reaches lower distortion only at markedly larger rates. These are complete-codec comparisons at nearby rates, not component ablations. Full rate–distortion curves are in Appendix \ref{app:exp1}.

\paragraph{Does the learned orbit geometry itself matter?}
\label{sec:learned-geometry-ablation}
The previous comparison cannot by itself establish that the \emph{learned} orbit geometry is useful. We therefore isolate scaffold discovery directly, before additive quantization or sparse realization can confound the comparison. In every chart we construct an orbit scaffold of rank \(r=0.75M_\ell\) from a small anchor set and compare it with \(20\) independently drawn, data-independent Gaussian subspaces of exactly the same rank. Both arms are evaluated on the complete local table \(Z^{(\ell)}\) through the relative projection error
$
\rho_\ell(S):=
\frac{\|(I-\Pi_S)Z^{(\ell)}\|_\mu}
{\|Z^{(\ell)}\|_\mu}.
$
Rank, ambient dimension, and evaluation data are therefore identical; the only advantage available to the orbit arm is that its geometry has been discovered from the data.

The separation is strikingly stable across all four embedding tables. Averaged over the four charts, the orbit versus Gaussian residuals are \(0.4775\) versus \(0.5011\) on GPT-2, \(0.4762\) versus \(0.4994\) on Llama-2-7B, \(0.4741\) versus \(0.4993\) on Mistral-7B, and \(0.4731\) versus \(0.4993\) on Mistral-7B-Instruct. These are reductions of \(4.7\%\), \(4.6\%\), \(5.1\%\), and \(5.3\%\) in relative projection error, or approximately \(9.2\%\), \(9.1\%\), \(9.9\%\), and \(10.2\%\) in squared unexplained projection energy. More strongly, in every one of the \(16\) model--chart comparisons, the orbit scaffold outperforms \emph{every one} of the \(20\) corresponding Gaussian draws. Since the rank is held fixed, this ablation shows that the gain is not explained by subspace dimension alone: the orbit construction discovers directions systematically better aligned with the full embedding table than an unstructured subspace of identical capacity.

\paragraph{Mechanism ablations.}
\label{sec:mechanism-ablations}

\paragraph{Post-gluing refinement improves the actual codec at fixed rate.}

We first disable and enable the gluing-aware refinement while keeping the complete decoder architecture, number of codewords, precision, and token indices budget fixed. Table~\ref{tab:gluing-ablation} reports the weighted reconstruction error in the PCA space immediately before and after refinement. The error decreases on every model, by $2.8$--$4.1\%$, without adding a single decoder bit. Of 280 proposed chart-stage updates, 265 are accepted for GPT-2 (15 rejected) and 280 for each 7B run (0 rejected). This is the empirical counterpart of the post-gluing refinement analysis: the completed global reconstruction provides a useful target for revisiting codewords that were originally fitted sequentially.

\begin{table}[t]
\centering
\small
\setlength{\tabcolsep}{4pt}
\caption{Gluing-aware refinement. Rate is unchanged; lower absolute weighted PCA-space error is better.}
\label{tab:gluing-ablation}
\begin{tabular}{lccc}
\hline
Model & Before & After & Reduction \\
\hline
GPT-2 & $0.4750$ & $0.4555$ & $4.11\%$ \\
Llama-2-7B & $0.3050$ & $0.2966$ & $2.76\%$ \\
Mistral-7B & $0.04777$ & $0.04643$ & $2.82\%$ \\
Mistral-Inst. & $0.04968$ & $0.04830$ & $2.77\%$ \\
\hline
\end{tabular}
\end{table}

\paragraph{What does redundancy actually do?}
To isolate the mechanism in Theorem~\ref{thm:redundant orbit add comprssion decomposition} from the particular standard chart construction, this ablation alone uses a richer independent-view family. We compare the nonredundant product boundary $A=1$ with $A=2$ and $A=3$, generated from one, two, and three independent orthogonal decompositions. We match total chart-stage capacity by fixing $LT=48$:
\[
(A,L,T)=(1,2,24),\qquad(2,4,12),\qquad(3,6,8).
\]
Consequently, all three arms have exactly the same model-specific rate. Table~\ref{tab:redundancy-ablation} reports the post-gluing PCA-space error and the fraction $\varepsilon_{\rm inc}^2/\varepsilon_{\rm loc}^2$ of local squared error removed by canonical gluing.

\begin{table*}[t]
\centering
\small
\setlength{\tabcolsep}{5pt}
\caption{Matched redundancy ablation. The total chart-stage budget $LT=48$ is fixed. ``Cancelled'' denotes $\varepsilon_{\rm inc}^2/\varepsilon_{\rm loc}^2$. GPT-2 uses $441.56$ bits/token in every arm; the $7$B models use $2890.14$ bits/token.}
\label{tab:redundancy-ablation}
\begin{tabular}{lccccc}
\hline
& \multicolumn{3}{c}{Post-gluing error} & \multicolumn{2}{c}{Cancelled local error} \\
Model & $A=1$ & $A=2$ & $A=3$ & $A=2$ & $A=3$ \\
\hline
GPT-2 & $0.18893$ & $0.18903$ & $0.18897$ & $98.85\%$ & $99.43\%$ \\
Llama-2-7B & $0.22596$ & $0.22378$ & $0.22057$ & $81.27\%$ & $89.85\%$ \\
Mistral-7B & $0.03577$ & $0.03492$ & $0.03430$ & $82.41\%$ & $90.66\%$ \\
Mistral-Inst. & $0.03968$ & $0.03813$ & $0.03814$ & $77.48\%$ & $87.04\%$ \\
\hline
\end{tabular}
\end{table*}

The nonredundant $A=1$ control has essentially no inconsistency component, as predicted. Once redundant views are introduced, canonical gluing removes a large component of the accumulated local error. This translates into a $2.39\%$ reduction in global PCA-space error on Llama-2-7B, a $4.09\%$ reduction on Mistral-7B, and a $3.91\%$ reduction on Mistral-Instruct at the best redundant setting. GPT-2 is informative in the opposite direction: despite cancellation above $98\%$, its final error is essentially unchanged because the local error itself grows under the matched-capacity constraint. This is precisely the tradeoff exposed by Theorem~\ref{thm:redundant orbit add comprssion decomposition}: redundancy does not guarantee improvement merely through the factor $A^{-1/2}$; its benefit depends jointly on the local approximation error and the inconsistency component removed by gluing. Numerically, the theoretical identity matches the measured global error to approximately $10^{-9}$ or better in every redundancy run.

\paragraph{Compilation removes construction state without changing reconstruction.}
Finally, we test whether the storage law in Section~\ref{sec:redundant orbit based chart} corresponds to a genuine deployable decoder. We reconstruct each table using only the objects charged in $C_{\rm OrBIT}$: the representation map, chart-system information, shared quantized codewords and scales, metadata, and token indices. The decoder is given no $W_\ell$, interpolants $J_{\ell,k}$, orbit dictionaries $\mathcal D_\ell$, anchors, scaffolds, or CCDS state. Across all four models, the decoder-only chart-space, PCA-space, and original-space reconstruction gaps are all $0$ at the reported numerical precision. Thus compilation is reconstruction-preserving in these runs.

The storage effect is also nontrivial. On GPT-2, retaining the explicit orbit-construction state would require $355.55$ bits/token, whereas the compiled decoder requires $324.15$ bits/token, an $8.83\%$ reduction. On each $7$B table the corresponding rates are $2889.13$ and $2739.97$ bits/token, a $5.16\%$ reduction. The orbit machinery therefore determines \emph{how} the shared continuous objects are constructed, but it is not hidden state required to decode them.


Taken together, these experiments separate the contributions of the framework rather than relying only on the final rate--distortion curve: OrBIT improves over the closest unconstrained additive comparator AQ at the codec level, while the rank-matched scaffold diagnostic shows directly that the learned orbit geometry captures more of the full embedding table than unstructured geometry of identical dimension; post-gluing refinement reduces the true reconstruction objective without increasing rate; redundant synthesis exhibits the cancellation term predicted by the theory, with a data-dependent net benefit; and compilation eliminates the orbit-construction state while reproducing exactly the transmitted decoder output. The resulting evidence supports the intended role of OrBIT as a structure-guided codec: the same globally glued reconstruction error determines the refinement targets, while the scaffold identifies the correctable directions and bounds the loss from geometric restriction.
Complete experimental configurations, full mechanism ablations, anchor-budget diagnostics, and three-seed robustness statistics are provided in Appendix~\ref{app:experimental-details}.


\section*{Reproducibility Statement}
Complete proofs of the theoretical results are provided in the appendix.
The construction of the chart system, orbit scaffolds, shared codebooks,
nested allocation, refinement procedure, compilation scheme, and
experimental protocol are described in the main text and appendix.
Hyperparameters and storage accounting are reported explicitly.


\section*{Ethics Statement}
This work develops theoretical and computational methods for neural embedding compression and does not involve human subjects or the collection of personal or sensitive data. The experiments use existing pretrained language models and are intended to study compression mechanisms rather than downstream deployment or decision-making. We are not aware of specific ethical risks introduced by the proposed methodology beyond those already associated with the underlying models on which it may be applied.

\section*{AI Use Statement}
Generative AI tools were used to assist with manuscript
editing and presentation, including improving clarity and organization.

 \bibliography{iclr2027_conference}
\bibliographystyle{iclr2027_conference}


\appendix

\section{Normalization and interpretation of the tight-frame constant}
\label{app:normalization redundancy}
The general local-to-global theory does not require normalized chart maps. However, the tight-frame constant $A$ measures genuine redundancy only after fixing their scale, since replacing every $R_\ell$ by $cR_\ell$ changes $A$ to $c^2A$. We therefore use the unit-gain normalization $R_\ell R_\ell^*=I_{\mathcal Z_\ell}$. Then $R_\ell^*R_\ell$ is the orthogonal projector onto the subspace observed by chart $\ell$, and for an $A$-tight system,
\begin{equation}
\sum_{\ell=1}^L R_\ell^*R_\ell=A I_H,
\qquad
A=\frac{\sum_{\ell=1}^L\dim\mathcal Z_\ell}{\dim\mathcal H}.
\end{equation}
Thus $A>1$ reflects genuine overlap of the normalized charts, and the canonical gluing map has norm $\|\Gamma_A\|=A^{-1/2}<1$. By contrast, an orthogonal product decomposition is nonredundant and its exact concatenation map has norm one. This normalization is therefore unnecessary for Theorem~\ref{thm:redundant orbit add comprssion decomposition} itself, but is essential when interpreting $A$ as redundancy and $A^{-1/2}$ as contractive synthesis.


\section{FINITE-DIMENSIONAL ENCODER AND IMPLEMENTATION DETAILS}
\label{app:implementation}

This appendix gives the orbit ingredients required to implement OrBIT without consulting the general learning framework. Compression uses only finite-table identity realization: local embeddings serve as both inputs and targets, and orbit dictionaries realize shared codeword prototypes. No predictor for unseen inputs is learned. The wrapped shift, block interpolation, and CCDS mechanism are inherited from \citet{puig2026orbit}; shared codebooks, redundant synthesis, residual-directed scheduling, and finite-decoder refinement define their use here.


In a chart, write the anchor vectors as \(x_j\). Choose orbit depths for which the columns of \(Z_k=[W^{L_j}x_j]\) are independent, and let \(Y_k=[x_j]\). Then \(J_k=Y_kZ_k^\dagger\) satisfies \(J_kW^{L_j}x_j=x_j\). The atoms \(J_kW^m x_j\), \(0\le m<M\), form a finite dictionary \(\mathcal D\). Since \(W^M=(\prod_iw_i)I\), deeper iterates add no new directions. Its full span equals the anchor span; changing the shift changes the dictionary representation inside that span.

 For a prototype \(v\), realization first targets \(p=\Pi_Sv\) and then constructs \(\widetilde c\in S\) using dictionary atoms. Orthogonality separates the two errors:
 \[
 \|v-\widetilde c\|^2=\|(I-\Pi_S)v\|^2+\|\Pi_Sv-\widetilde c\|^2.
 \]
 The first is a coverage limitation; the second depends on the realization rule and budget. The final codeword is \(c=Q_{b_\gamma}(\widetilde c)\), and all assignments and acceptance tests use this compiled vector.

\paragraph{CCDS mechanism.}Let \(\widehat p_0=0\), \(r_0=p\), and \(V_0=\{0\}\). At realization step \(k\), select a dictionary atom outside \(V_{k-1}\), orthogonalize it to obtain a unit innovation \(u_k\), and put \(V_k=V_{k-1}\oplus\operatorname{span}\{u_k\}\). Set
 \[
 a_k=(I-\Pi_{V_{k-1}})r_{k-1},\quad
 \alpha_k=\langle a_k,u_k\rangle,\quad
 d_{k,\min}^2=\|a_k\|^2-\alpha_k^2.
 \]
 For a feasible deviation \(d_k\ge d_{k,\min}\), choose
 \[
 \lambda_k=\alpha_k\pm\sqrt{d_k^2-d_{k,\min}^2},\qquad
 \widehat p_k=\widehat p_{k-1}+\lambda_ku_k.
 \]
 With \(r_k=p-\widehat p_k\), this enforces \(\operatorname{dist}(r_k,V_{k-1})=d_k\). When every step uses \(d_k=d_{k,\min}\), it recovers orthogonal projection along the selected independent atoms; interior choices control a different realization path.

 In the reported encoder we use boundary full CCDS: \(d_k=d_{k,\min}\), hence \(\lambda_k=\alpha_k\). Dictionary atoms are unit-normalized; zero atoms are discarded, and a candidate must contribute an innovation of norm greater than \(10^{-7}\). At each step the unselected admissible atom maximizing \(|\langle r_{k-1},q_j\rangle|/\|(I-\Pi_{V_{k-1}})q_j\|\) is selected and reorthogonalized. The solver stops at \(\|r_k\|\le10^{-2}\|p\|\), at support 32 (GPT-2) or 96 (7B), or when no further independent innovation is available; the resulting selected-support least-squares realization is then used.

 \paragraph{Shared representation and standard chart system.}
The embedding table is first centered and mapped by the shared PCA representation used throughout the experiments. We use $d=64$ for GPT-2 and $d=192$ for the $4096$-dimensional Llama and Mistral embedding tables. The corresponding chart dimensions are $M_\ell=M=32$ and $M_\ell=M=96$, respectively. The PCA mean and loading matrix are part of $C_{\rm rep}$ and are included in all reported rates. From a seeded 12,000-row vocabulary sample, a seeded permutation selects the first \(D=10{,}000\) rows as the table compressed in the reported experiments; the remaining 2,000 rows are unused. PCA is fitted only on these \(10{,}000\) rows after subtracting their empirical mean, and decoding to the original space uses the stored loading matrix followed by addition of that mean.

Except in the redundancy ablation described below, the experiments use $L=4$ charts and $A=2$. A single stored seed generates a Gaussian matrix whose orthogonal factor is denoted by $Q\in\mathbb R^{d\times d}$. Concretely, the run seed initializes a CPU generator, a \(d\times d\) standard Gaussian matrix is drawn in float64, its QR \(Q\)-factor is cast to float32, and the chart windows are then fixed deterministically by their cyclic offsets. For each chart $\ell$, a coordinate selector $P_\ell$ extracts a cyclic window of $M=d/2$ rows, with the windows placed so that every ambient coordinate is covered twice, and
$$
R_\ell:=P_\ell Q:\mathcal H\to\mathcal Z_\ell.
$$
The resulting system is verified numerically to satisfy
$$
R_\ell R_\ell^*=I_{\mathcal Z_\ell},
\qquad
\sum_{\ell=1}^L R_\ell^*R_\ell=2I_{\mathcal H}.
$$
Thus the standard experimental configuration uses canonical gluing $\Gamma_A=A^{-1}R^*$ with $A=2$.

\paragraph{Anchor candidates and block construction.}
For each chart $\ell$, a candidate anchor set $I_\ell$ contains a prescribed fraction $p$ of the $D$ table rows. The encoder evaluates up to five generated candidate anchor sets, until the first accepted one reproducibly. Starting from a seeded initial token, anchors are added by farthest-point sampling in $\mathcal Z_\ell$: each new anchor maximizes its minimum Euclidean distance from the anchors already selected. Candidate sets are kept disjoint whenever enough unused rows remain.

Each candidate $I_\ell$ is partitioned into blocks
$
I_\ell=I_{\ell,1}\cup\cdots\cup I_{\ell,K_\ell}.
$
Starting from a seeded ordering, the encoder searches over depths
$
L_{\ell,j}\in\{0,\ldots,M_\ell-1\}
$
and admits an anchor to the current block only when its orbit representative is numerically independent of the representatives already present. If $v$ is the candidate representative and $P$ is the orthogonal projector onto the current span, the numerical independence test is
$$
\frac{\|(I-P)v\|}{\|v\|}>0.1.
$$
The partition is refined until every block contains at most \(\lceil M_\ell/3\rceil\) anchors.
This refinement is purely an encoder-side numerical procedure; the resulting partition is used only to construct the local orbit scaffold.

\paragraph{Learning the wrapped shift and fixing the interpolants.}
For a candidate partition, the wrapped-shift weights are parameterized as
$$
w_{\ell,j}
=
\exp\left(
s_{\ell,j}-\overline{s}_\ell
+\frac{1}{M_\ell}\log\lambda
\right),
\qquad
\overline{s}_\ell
=
\frac{1}{M_\ell}\sum_{j=0}^{M_\ell-1}s_{\ell,j},
\qquad
\lambda=0.5.
$$
Hence
$
\prod_{j=0}^{M_\ell-1}w_{\ell,j}=\lambda
$
throughout optimization. During shift learning the orbit atoms are row-normalized. Writing \(G=A^\top A\), the differentiable ridge surrogate is \(P_{\rm rid}=(G+\lambda_{\rm rid}I)^{-1}G\), with \(\lambda_{\rm rid}=10^{-4}\max_i G_{ii}\), and the encoder minimizes the weighted mean \(\sum_n\mu_n\|z_n-z_nP_{\rm rid}\|^2/\sum_n\mu_n\). Adam is run for 120 steps at learning rate \(0.05\); when the local table exceeds 2048 rows, a seeded 2048-row subsample is used. The final \(J_{\ell,k}\) are subsequently recomputed as \(Y_{\ell,k}Z_{\ell,k}^\dagger\). Because exact interpolation fixes \(S_\ell=\operatorname{span}\mathcal D_\ell\) to the anchor-target span, this optimization does not change the exact scaffold subspace; it changes the orbit dictionary geometry and conditioning, and hence finite-budget realization within that span.

After $W_\ell$ has been learned, the depths $L_{\ell,j}$ are recomputed greedily so that the representatives
\[
z_j^{(\ell)}
=
W_\ell^{L_{\ell,j}}R_\ell e_j,
\qquad
j\in I_{\ell,k},
\]
remain numerically independent inside every block. Writing these representatives and their targets as the columns of $Z_{\ell,k}$ and $Y_{\ell,k}$, respectively, the implementation forms
$
J_{\ell,k}=Y_{\ell,k}Z_{\ell,k}^{\dagger}
$
and then constructs $\mathcal D_\ell$ and $S_\ell=\operatorname{span}\mathcal D_\ell$ exactly as defined in Section~2. The block interpolation residuals are checked numerically, with tolerance $10^{-6}$.

Crucially, scaffold acceptance is not performed on the anchor set. For every candidate, the encoder evaluates the projection error on the complete local table as indicated in (\ref{eq:scaffold-accept}).
The default threshold is
$
\tau_{\rm proj}=10^{-3}.
$
Candidates are evaluated in seeded order and the first scaffold satisfying the full-table criterion is retained. If none satisfies the threshold, the run is rejected rather than relaxing $\tau_{\rm proj}$.

\paragraph{Orbit realization and finite-bit compilation.}
Once $\mathcal D_\ell$ and $S_\ell$ have been fixed, every shared prototype $p\in\mathcal Z_\ell$ is realized through the same encoder-side map. The target component represented by the scaffold is passed to CCDS over $\mathcal D_\ell$, producing
$$
\widetilde c
=
\operatorname{CCDS}_{\mathcal D_\ell}(p)
\in S_\ell.
$$
The support cap is $32$ for the $M=32$ charts and $96$ for the $M=96$ charts.
Compilation uses a symmetric per-codeword quantizer. For \(\widetilde c\), let \(s\) be \(\|\widetilde c\|_\infty\) clamped to the smallest positive fp16 normal value and stored in fp16, and let \(L=2^{b_\gamma-1}-1\). The compiled vector is
\[
c=\frac{s}{L}\operatorname{clip}\!\left(\operatorname{round}\!\left(\frac{L\widetilde c}{s}\right),-L,L\right).
\]We use \(b_\gamma=8\), so \(L=127\), and store one 16-bit scale per codeword. A zero codeword remains identically zero. All subsequent assignments and distortion calculations use this compiled vector.

All subsequent assignments, residual updates, refinement steps, and distortion measurements use the resulting compiled codeword $c$, not the unquantized orbit realization $\widetilde c$. Thus the reported distortion always corresponds to the finite decoder state that is actually charged in the storage ledger.

The nested shared-codebook construction in the main experiments uses
$$
T=7,
\qquad
Q_{\ell,t}=Q=32,
\qquad
\left\lceil\log_2Q\right\rceil=5.
$$
Each newly allocated stage is initialized by 20 seeded k-means iterations on its current global-residual target, followed by within-stage centroid, realization, compilation, and reassignment sweeps, and is then committed to the decoder. Post-gluing refinement performs 10 sweeps over chart-stage pairs: for each stage its codewords are recomputed from the current assignments, realized and compiled, the assignments are recomputed against the compiled codewords, and the whole stage update is retained only when the current global reconstruction error does not increase (tolerance \(10^{-12}\)); otherwise the previous stage state is kept.

\paragraph{Fixed numerical settings.}
Unless a specific ablation changes one of them explicitly, the reported experiments use the same numerical settings:
$$
\lambda=0.5,
\qquad
\tau_{\rm proj}=10^{-3},
\qquad
Q=32,
\qquad
b_\gamma=8,
\qquad
b_s=16,
\qquad
T=7,
$$
with $120$ Adam steps at learning rate $0.05$ for wrapped-shift learning, at most $2048$ local rows in that optimization, $20$ seeded $k$-means iterations for residual-prototype initialization, $10$ complete post-gluing refinement sweeps, and seed $13$ for the single-seed runs. All rate computations include the shared representation map, chart-system information, every quantized shared codeword and scale, fixed codec metadata, and all token indices.

The complete encoder proceeds by first learning the shared representation and constructing the chart system, then discovering the orbit scaffolds and using them to build the nested shared codebooks. The resulting codebooks are subsequently refined with respect to the globally glued reconstruction and finally compiled into the finite decoder. Only this final compiled state is retained at decoding time.


\section{ADDITIONAL EXPERIMENTAL DETAILS AND RESULTS}
\label{app:experimental-details}

This appendix complements the main experimental section with the complete protocol and the diagnostics used to isolate the individual components of OrBIT. The purpose is not to repeat the main rate--distortion discussion, but to make explicit what is held fixed in each comparison and to report the measurements underlying the mechanism claims. Experiments~1--5 use the final codec configurations reported in the main text. Experiment~6 studies the amount of data required to construct an accepted scaffold, and Experiment~7 tests the stability of the final headline configuration across independent seeds.

\subsection{Common protocol and implementation checks}
\label{app:experimental-protocol}

For every embedding model we extract $12{,}000$ token embeddings and compress a table of
\[
D=10{,}000
\]
embeddings under uniform table weights. GPT-2 embeddings have original dimension $768$ and are reduced to $d=64$; the three $7$B embedding tables have original dimension $4096$ and are reduced to $d=192$. The representation map is fitted on the table being compressed and its storage is included in every reported rate. The corresponding shared representation-map costs are $1{,}597{,}440$ bits for GPT-2 and $25{,}296{,}896$ bits for each $4096$-dimensional table.

The standard OrBIT configuration uses an exact $A=2$ tight chart system with $L=4$ charts, with $M_\ell=32$ for GPT-2 and $M_\ell=96$ for the $7$B models. For each chart, up to five reproducibly generated anchor sets are considered, and candidate scaffolds are evaluated in seeded order; the first satisfying the full-table acceptance criterion
\[
\varepsilon_{\rm proj}\leq \tau_{\rm proj},
\qquad
\tau_{\rm proj}=10^{-3},
\]
is retained. The standard anchor budget is $0.75\%$ for GPT-2 and $1.2\%$ for the $7$B tables. At these budgets the accepted scaffolds attain full chart rank, so the projection bottleneck becomes negligible while the remaining orbit construction and realization mechanisms remain active. Experiment~6 characterizes the transition to this full-coverage regime and verifies that the standard budgets are conservative.

Unless otherwise stated, the compiled codec uses $T=7$ additive stages, $Q_{\ell,t}=32$ codewords per chart-stage pair, $5$ bits per token index, $8$-bit shared codeword coordinates, and a maximum CCDS support of $32$ for GPT-2 and $96$ for the $7$B tables. The common representation map, chart-system description, quantized shared codewords and scales, metadata, and all token indices are included in the reported rate.

Table~\ref{tab:appendix-standard-config} summarizes the standard operating point used throughout Experiments~1--3 and~5 and as the headline configuration in Experiment~7.

\begin{table*}[t]
\centering
\small
\setlength{\tabcolsep}{5pt}
\caption{Standard OrBIT configuration. The anchor fraction is measured relative to the $D=10{,}000$-token table. The reported rate is the complete compiled rate, including the shared representation map.}
\label{tab:appendix-standard-config}
\begin{tabular}{lcccccccc}
\hline
Model & Original dim. & $d$ & $L$ & $M_\ell$ & Anchor frac. & Max support & $T$ & Bits/token \\
\hline
GPT-2 & $768$ & $64$ & $4$ & $32$ & $0.75\%$ & $32$ & $7$ & $324.1472$ \\
Llama-2-7B & $4096$ & $192$ & $4$ & $96$ & $1.2\%$ & $96$ & $7$ & $2739.968$ \\
Mistral-7B & $4096$ & $192$ & $4$ & $96$ & $1.2\%$ & $96$ & $7$ & $2739.968$ \\
Mistral-7B-Instruct & $4096$ & $192$ & $4$ & $96$ & $1.2\%$ & $96$ & $7$ & $2739.968$ \\
\hline
\end{tabular}
\end{table*}

All chart systems are numerically certified before compression. In the standard $A=2$ runs, the maximum observed deviations from unit gain and tightness are on the order of $10^{-8}$, and the value of $A$ recovered from the frame trace agrees with $2$ to numerical precision. The same certification is performed separately for the $A=1,2,3$ systems in the redundancy ablation.

\paragraph{Baseline accounting.}
Experiment~1 compares OrBIT with PQ, OPQ, AQ, scalar quantization, and truncated SVD. For PQ, OPQ, and AQ we sweep the codebook configuration and index precision. Their shared codebooks, transforms, the common representation map, and token indices are included in the rate. We additionally report PQ-q8, OPQ-q8, and AQ-q8, in which the shared continuous codewords use the same $8$-bit-coordinate and $16$-bit-scale convention as OrBIT and reconstruction is recomputed after quantization. Hence the matched-precision comparison does not grant OrBIT an advantage from cheaper shared-codeword storage. The complete rate--distortion curves shown in the main text use the actual measured decoder output at every operating point.

\subsection{Experiment 1: additional rate--distortion details}
\label{app:exp1}

The OrBIT rate--distortion curve is obtained by retaining the same learned chart geometry and progressively decoding the first
\[
T=1,\ldots,7
\]
additive stages. Each point is therefore an actual truncated compiled decoder rather than an interpolation between measured operating points. The resulting prefix curve is empirically monotone on all four tables; monotonicity is observed rather than guaranteed by the refinement procedure. Figure~1 shows the complete rate--distortion curves.

\begin{figure*}[t]
    \centering
    \includegraphics[width=\textwidth]{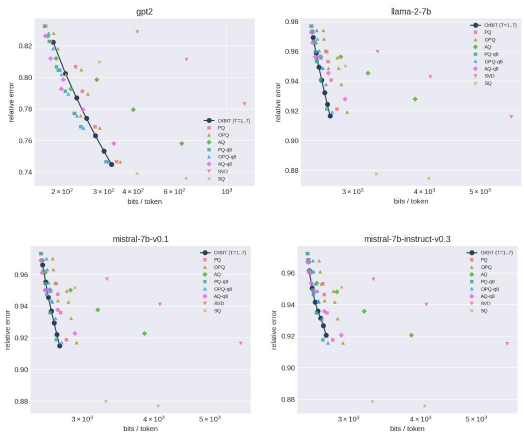}
    \caption{
    Rate--distortion curves across the four embedding tables.
    Rate is measured in bits per token and distortion by relative reconstruction
    error in the original embedding space. Lower values are better on both axes.
    }
    \label{fig:rate_distortion}
\end{figure*}

At \(T=7\), the complete compiled rates correspond to compression ratios of \(37.91\times\) for GPT-2 and \(23.92\times\) for each 7B embedding table relative to 16-bit storage of the original embeddings. The representation map is included in these ratios; in particular, the large \(4096\rightarrow192\) PCA map is not treated as free shared side information.

The baselines separate two conceptually different comparisons. PQ and OPQ impose a prescribed product geometry, while AQ is the closest baseline to OrBIT at the level of the coding mechanism: both represent a vector through a sum of entries selected from shared additive codebooks. The crucial difference is what codewords are admissible. AQ optimizes shared codewords without an orbit-scaffold restriction. OrBIT instead discovers a local scaffold \(S_\ell\) from a small anchor set through the wrapped-shift and interpolation construction, uses that geometry to constrain the shared codewords, distributes the fixed stage budget through the residual of the globally glued reconstruction, and combines the resulting local codes through redundant synthesis. Thus the OrBIT--AQ comparison asks whether adding discovered geometric structure to an additive-codebook architecture is beneficial end-to-end; it is not intended as a one-component ablation, since OrBIT also contains the chart, allocation, and gluing mechanisms isolated separately elsewhere.

This comparison is particularly favorable to the structured construction. Under the matched q8 accounting, OrBIT and AQ-q8 give respectively \(324.15/0.7449\) versus \(331.88/0.7582\) bits/token / relative error on GPT-2, \(2739.97/0.9163\) versus \(2911.54/0.9278\) on Llama-2-7B, and \(2739.97/0.9151\) versus \(2911.54/0.9228\) on Mistral-7B. Hence OrBIT is strictly better in both rate and distortion on all three tables. On Mistral-7B-Instruct the distortions are essentially tied at \(0.9205\), but OrBIT uses \(2739.97\) rather than \(2911.54\) bits/token, a rate reduction of approximately \(5.9\%\). This is notable precisely because AQ does not impose the orbit admissibility constraint on its shared codewords: the structured geometry is not purchased by sacrificing the rate--distortion behavior of the additive representation.

PQ and OPQ then provide a distinct and empirically harder test. They exploit a strong Cartesian product structure rather than unrestricted additive codewords. At the representative q8 operating points, OrBIT gives lower distortion than both PQ-q8 and OPQ-q8 on GPT-2, Llama-2-7B, and Mistral-7B at nearby rates. On Mistral-7B-Instruct, PQ-q8 and OPQ-q8 instead achieve lower distortion. This counterexample is useful: the experiments do not support a claim of universal dominance, but they show that a codec whose shared continuous objects are restricted by discovered orbit geometry can compete directly with mature product-quantization methods while substantially improving over the closest unconstrained additive comparator.

The matched q8 baselines also separate coding geometry from storage precision. Standard PQ, OPQ, and AQ measure performance under their usual storage representations, whereas PQ-q8, OPQ-q8, and AQ-q8 quantize their shared continuous codewords with the same 8-bit-coordinate and 16-bit-scale convention used by OrBIT and recompute the reconstruction after quantization. Consequently, the favorable OrBIT comparisons cannot be attributed simply to charging its shared continuous objects at a cheaper precision.

\subsection{Experiment 2: learned orbit geometry versus a rank-matched Gaussian control}
\label{app:exp2}

Experiment~2 isolates the geometric contribution of the orbit construction itself. Unlike the complete codec experiments, this diagnostic is performed directly at the scaffold level: no additive quantizer, shared-codebook optimization, or sparse realization solver is used in either arm. For each chart \(\ell\), we construct a data-aligned orbit scaffold \(S_\ell^{\mathrm{orb}}\) of rank
$
r=0.75M_\ell
$
and compare it with \(20\) independent rank-\(r\) Gaussian subspaces \(S_{\ell,j}^{\mathrm{gauss}}\). The Gaussian matrices are orthonormalized before evaluation, so the controls are unstructured random subspaces independent of the embedding table. The orbit and Gaussian arms therefore have identical rank and ambient chart dimension; what differs is whether their geometry has been discovered from the data through the orbit construction.

The orbit scaffold is deliberately constructed from a small anchor budget and then evaluated on the \emph{complete} chart table. GPT-2 uses \(24\) anchors per chart, corresponding to \(0.24\%\) of the \(10{,}000\)-token table and rank \(24\) in \(M_\ell=32\); the three \(7\)B tables use \(72\) anchors, corresponding to \(0.72\%\) of the table and rank \(72\) in \(M_\ell=96\). For a candidate rank-\(r\) scaffold \(S\), we report
\[
\rho_\ell(S):=
\frac{\|(I-\Pi_S)Z^{(\ell)}\|_\mu}
{\|Z^{(\ell)}\|_\mu}.
\]
Lower values mean that the rank-constrained geometry captures more of the full local embedding table.

The Gaussian control also provides a useful calibration. For a data-independent random rank-\(r\) orthogonal projector \(P\) in dimension \(M_\ell\),
\[
\mathbb{E}
\frac{\|(I-P)Z^{(\ell)}\|_\mu^2}
{\|Z^{(\ell)}\|_\mu^2}
=
1-\frac{r}{M_\ell}.
\]
Since \(r/M_\ell=0.75\), the expected squared unexplained energy is \(0.25\), corresponding to a root-residual scale of \(0.5\). The Gaussian controls in Table~\ref{tab:appendix-geometry-ablation} concentrate almost exactly at this value, confirming that they behave as the intended rank-matched unstructured null.

\begin{table*}[t]
\centering
\small
\setlength{\tabcolsep}{5pt}
\caption{Experiment~2: data-aligned orbit scaffold versus a rank-matched Gaussian control. Each Gaussian entry is the mean \(\pm\) standard deviation over \(20\) independent random subspaces; ``Best Gaussian'' is the smallest residual among those \(20\) draws. All methods have rank \(r=0.75M_\ell\), and evaluation uses the complete local table. Lower relative projection error is better.}
\label{tab:appendix-geometry-ablation}
\begin{tabular}{lcccc}
\hline
Model & Chart & Orbit scaffold & Gaussian control & Best Gaussian \\
\hline
GPT-2 & $1$ & $0.47290$ & $0.49930 \pm 0.00641$ & $0.48787$ \\
      & $2$ & $0.47890$ & $0.50236 \pm 0.00688$ & $0.48572$ \\
      & $3$ & $0.48054$ & $0.50028 \pm 0.00525$ & $0.49094$ \\
      & $4$ & $0.47782$ & $0.50233 \pm 0.00623$ & $0.49177$ \\
\hline
Llama-2-7B & $1$ & $0.47805$ & $0.49885 \pm 0.00202$ & $0.49489$ \\
            & $2$ & $0.47321$ & $0.49905 \pm 0.00295$ & $0.49120$ \\
            & $3$ & $0.47674$ & $0.49941 \pm 0.00317$ & $0.49518$ \\
            & $4$ & $0.47689$ & $0.50017 \pm 0.00235$ & $0.49676$ \\
\hline
Mistral-7B & $1$ & $0.47625$ & $0.49872 \pm 0.00234$ & $0.49409$ \\
           & $2$ & $0.47079$ & $0.49893 \pm 0.00335$ & $0.49007$ \\
           & $3$ & $0.47601$ & $0.49944 \pm 0.00360$ & $0.49485$ \\
           & $4$ & $0.47332$ & $0.50018 \pm 0.00267$ & $0.49660$ \\
\hline
Mistral-7B-Instruct & $1$ & $0.47290$ & $0.49872 \pm 0.00236$ & $0.49405$ \\
                    & $2$ & $0.47227$ & $0.49894 \pm 0.00338$ & $0.49004$ \\
                    & $3$ & $0.47340$ & $0.49942 \pm 0.00361$ & $0.49475$ \\
                    & $4$ & $0.47372$ & $0.50018 \pm 0.00271$ & $0.49660$ \\
\hline
\end{tabular}
\end{table*}

The result is stronger than a comparison of averages. In all \(16\) model--chart pairs, the learned orbit scaffold has lower projection error than the \emph{best} of the \(20\) independently drawn Gaussian controls. Hence every orbit scaffold beats every corresponding random draw, for \(320\) rank-matched Gaussian controls in total. At the model level, averaging across charts gives
\[
\begin{array}{lccc}
\hline
\text{Model} & \text{Orbit} & \text{Gaussian} & \text{Relative reduction} \\
\hline
\text{GPT-2} & 0.47754 & 0.50107 & 4.70\% \\
\text{Llama-2-7B} & 0.47622 & 0.49937 & 4.64\% \\
\text{Mistral-7B} & 0.47409 & 0.49932 & 5.05\% \\
\text{Mistral-7B-Instruct} & 0.47307 & 0.49932 & 5.26\% \\
\hline
\end{array}
\]
and, because \(\rho_\ell^2\) is the fraction of local squared energy left unexplained, these improvements correspond to reductions of approximately \(9.17\%\), \(9.06\%\), \(9.85\%\), and \(10.24\%\), respectively, in squared unexplained projection energy.

This experiment therefore isolates the role assigned to orbit geometry in OrBIT. A rank-\(0.75M_\ell\) subspace already has substantial approximation power purely by dimension, which is why the random controls lie near \(0.5\). The nontrivial observation is that the orbit construction uses only a small anchor set yet consistently places those same \(r\) dimensions in directions that explain more of the \emph{full} chart table. The advantage is reproduced across architectures, across all charts, and against every sampled rank-matched random control. Thus the scaffold is not functioning merely as a generic low-dimensional container: its data-aligned orbit structure measurably reduces the geometric projection bottleneck that subsequently constrains the shared OrBIT codewords.

Experiment~2 therefore isolates what Experiment~1 cannot. At rank \(0.75M_\ell\), an arbitrary subspace already captures approximately \(75\%\) of the squared energy simply by dimensionality; the relevant question is whether the same rank can be placed more intelligently. The orbit construction does so using less than \(1\%\) of the table as anchor targets: it consistently selects directions that leave about \(9\%\)--\(10\%\) less unexplained energy than the rank-matched random geometry. Combined with Experiment~1, the evidence has two levels: the scaffold diagnostic shows that orbit dynamics discover a better-aligned admissible geometry, while the end-to-end AQ comparison shows that imposing this geometry remains beneficial after the full representation is quantized, indexed, stored, and decoded.

\subsection{Experiment 3: gluing-aware refinement}
\label{app:exp3}

The refinement ablation compares the complete reconstruction immediately before and after post-gluing codebook refinement. The number of codewords, quantization precision, chart system, token indices, and total storage are identical in both arms. A proposed chart-stage update, including reassignment to the newly compiled codewords, is retained only when it decreases the current global reconstruction objective, so the ablation directly tests whether revisiting locally constructed codewords through the global residual is useful.

Table~\ref{tab:appendix-refinement} reports both the reduced-space objective used by the refinement and the final error after lifting back to the original embedding space. Every model improves at exactly the same bitrate.

\begin{table*}[t]
\centering
\small
\setlength{\tabcolsep}{4pt}
\caption{Experiment~3: post-gluing refinement. The bitrate is identical before and after refinement. The PCA-space columns report the absolute weighted reconstruction error optimized by refinement; the end-to-end column is measured after lifting to the original embedding space.}
\label{tab:appendix-refinement}
\begin{tabular}{lccccc}
\hline
Model & PCA before & PCA after & Reduction & End-to-end before/after & Accepted updates/280 \\
\hline
GPT-2
& $0.47502$ & $0.45549$ & $4.11\%$
& $0.74570/0.74493$ & $265$ \\
Llama-2-7B
& $0.30500$ & $0.29659$ & $2.76\%$
& $0.91865/0.91631$ & $280$ \\
Mistral-7B
& $0.04777$ & $0.04643$ & $2.82\%$
& $0.91737/0.91518$ & $280$ \\
Mistral-7B-Instruct
& $0.04968$ & $0.04830$ & $2.77\%$
& $0.92299/0.92048$ & $280$ \\
\hline
\end{tabular}
\end{table*}

The reduced-space improvement is substantially larger than the corresponding end-to-end movement because the latter also contains the fixed representation-truncation error. The relevant mechanism result is nevertheless unambiguous: once all local codebooks exist, the global residual exposes additional correctable error, and exploiting it improves the actual deployed reconstruction without spending additional bits.

\subsection{Experiment 4: matched-capacity redundancy audit}
\label{app:exp4}

The standard experiments use the $A=2$ chart family of Table~\ref{tab:appendix-standard-config}. Experiment~4 is deliberately different: it is a mechanism audit designed to make the redundancy term observable as cleanly as possible. Only for this ablation, we use independent orthogonal decompositions to construct increasingly redundant analysis systems and compare
\[
(A,L,T)=(1,2,24),\qquad
(2,4,12),\qquad
(3,6,8).
\]
Thus
\[
LT=48
\]
is fixed in all three arms. The nonredundant $A=1$ product boundary, the $A=2$ system, and the $A=3$ system therefore have exactly the same number of chart-stage codebooks and exactly the same bitrate: $441.5552$ bits/token for GPT-2 and $2890.144$ bits/token for each $7$B table.

For this audit we record
\[
\varepsilon_{\rm loc},
\qquad
\varepsilon_{\rm inc},
\qquad
\frac{\varepsilon_{\rm inc}^{2}}{\varepsilon_{\rm loc}^{2}},
\]
together with the measured global residual after canonical gluing. The third quantity is the fraction of accumulated local squared error removed by the inconsistency-cancellation term.

\begin{table*}[t]
\centering
\scriptsize
\setlength{\tabcolsep}{4pt}
\caption{Experiment~4: matched-capacity redundancy audit. ``Cancelled'' is $\varepsilon_{\rm inc}^{2}/\varepsilon_{\rm loc}^{2}$. The global residual is the weighted reduced-space norm entering the tight-frame identity. All three values of $A$ have identical bitrate within each model.}
\label{tab:appendix-redundancy-full}
\begin{tabular}{lcccccc}
\hline
Model & $A$ & $\varepsilon_{\rm loc}$ & $\varepsilon_{\rm inc}$ & Cancelled & Global residual & End-to-end rel. error \\
\hline
GPT-2
& $1$ & $0.18893$ & $3.38\times10^{-7}$ & $0.00\%$ & $0.18893$ & $0.73763$ \\
& $2$ & $2.48802$ & $2.47362$ & $98.85\%$ & $0.18903$ & $0.73764$ \\
& $3$ & $4.32155$ & $4.30914$ & $99.43\%$ & $0.18897$ & $0.73763$ \\
\hline
Llama-2-7B
& $1$ & $0.22596$ & $1.51\times10^{-7}$ & $0.00\%$ & $0.22596$ & $0.89909$ \\
& $2$ & $0.73135$ & $0.65933$ & $81.27\%$ & $0.22378$ & $0.89863$ \\
& $3$ & $1.19937$ & $1.13690$ & $89.85\%$ & $0.22057$ & $0.89796$ \\
\hline
Mistral-7B
& $1$ & $0.03577$ & $2.45\times10^{-8}$ & $0.00\%$ & $0.03577$ & $0.89990$ \\
& $2$ & $0.11775$ & $0.10689$ & $82.41\%$ & $0.03492$ & $0.89885$ \\
& $3$ & $0.19439$ & $0.18509$ & $90.66\%$ & $0.03430$ & $0.89810$ \\
\hline
Mistral-7B-Instruct
& $1$ & $0.03968$ & $2.28\times10^{-8}$ & $0.00\%$ & $0.03968$ & $0.90623$ \\
& $2$ & $0.11363$ & $0.10002$ & $77.48\%$ & $0.03813$ & $0.90394$ \\
& $3$ & $0.18355$ & $0.17124$ & $87.04\%$ & $0.03814$ & $0.90397$ \\
\hline
\end{tabular}
\end{table*}

The control behaves exactly as intended. At $A=1$, the inconsistency component is numerically zero. Once redundant views are introduced, a large part of the local error becomes mutually inconsistent and is removed by canonical synthesis: between $77.48\%$ and $98.85\%$ at $A=2$, and between $87.04\%$ and $99.43\%$ at $A=3$.

The ablation also exposes the second half of the theoretical mechanism rather than hiding it. Redundancy can increase the aggregate local error $\varepsilon_{\rm loc}$, so cancellation does not by itself guarantee a lower final residual. GPT-2 is the clearest example: despite cancellation above $98\%$, its final residual is essentially unchanged. On the larger embedding tables, the cancellation is sufficiently strong to reduce the global error: relative to $A=1$, the best redundant configuration lowers the measured reduced-space global residual by approximately $2.4\%$ on Llama-2-7B, $4.1\%$ on Mistral-7B, and $3.9\%$ on Mistral-7B-Instruct.

Finally, the measured identity underlying the redundancy decomposition is numerically exact. Across all twelve runs, the absolute discrepancy between the directly measured global residual and the value reconstructed from the local-error/cancellation decomposition is below $4\times10^{-9}$. The experiment therefore verifies both the presence of the cancellation mechanism and its precise quantitative accounting.

\subsection{Experiment 5: compilation and decoder-only audit}
\label{app:exp5}

The orbit construction contains objects that are useful to the \emph{encoder} but need not survive in the deployed decoder. Experiment~5 tests this distinction directly. After all shared codewords have been constructed and quantized, we reconstruct the complete table from the compiled payload alone. In particular, this decoder is not given the wrapped shifts $W_\ell$, interpolation maps $J_{\ell,k}$, orbit dictionaries $\mathcal D_\ell$, anchors, or CCDS construction state.

For every model, the decoder-only reconstruction agrees exactly with the reconstruction obtained before discarding the construction state. The maximum chart-space gap, maximum decoder output gap, relative decoder gap, and original-space relative gap are all $0$ at the reported numerical precision.

\begin{table*}[t]
\centering
\scriptsize
\setlength{\tabcolsep}{1pt}
\caption{Experiment~5: compilation audit. ``Explicit'' retains the orbit/scaffold construction state in addition to the compiled payload. ``Compiled'' retains only the objects required by the deployed decoder. Reconstruction is identical in the two cases.}
\label{tab:appendix-compilation}
\begin{tabular}{lccccc}
\hline
Model & Explicit bits/token & Compiled bits/token & Saved bits/token & Storage reduction & Reconstruction gap \\
\hline
GPT-2
& $355.5468$ & $324.1472$ & $31.3996$ & $8.83\%$ & $0$ \\
Llama-2-7B
& $2889.1328$ & $2739.9680$ & $149.1648$ & $5.16\%$ & $0$ \\
Mistral-7B
& $2889.1328$ & $2739.9680$ & $149.1648$ & $5.16\%$ & $0$ \\
Mistral-7B-Instruct
& $2889.1328$ & $2739.9680$ & $149.1648$ & $5.16\%$ & $0$ \\
\hline
\end{tabular}
\end{table*}

This audit is important for the interpretation of the storage law. The orbit machinery determines how the shared codewords are discovered and realized, but it is not unreported decoder state. Once compilation has taken place, the orbit construction can be discarded without changing a single reconstructed coordinate. The reported OrBIT rate therefore corresponds to a genuine finite transmitted representation rather than to a decoder that implicitly retains its training geometry.

\subsection{Experiment 6: anchor-budget and scaffold-rank diagnostic}
\label{app:exp6}

The main experiments intentionally use very small anchor subsets. Experiment~6 tests how much of the table is actually required before the orbit construction produces an accepted scaffold on the \emph{full} local table.

This diagnostic is run independently from the standard encoder wrapper so that no automatic lower bound on the anchor budget is imposed. For every requested budget, the orbit geometry is reconstructed from scratch and we report the resulting scaffold rank and canonical full-table projection error. Acceptance always uses the same threshold
\[
\tau_{\rm proj}=10^{-3}.
\]

The dimensional structure makes this test especially informative. In the finite-dimensional construction,
\[
\operatorname{ran}(J_{\ell,k})
\subseteq
\operatorname{span}(Y_{\ell,k}),
\]
so an anchor-generated scaffold cannot have rank larger than the number of available anchor targets. Hence a $96$-dimensional chart cannot be full rank with only $50$, $64$, or $75$ anchors. The experiment asks whether the construction actually reaches full chart rank once doing so becomes dimensionally feasible.

\begin{table*}[t]
\centering
\scriptsize
\setlength{\tabcolsep}{4pt}
\caption{Experiment~6: anchor-budget diagnostic. The rank column is the scaffold rank in a chart of dimension $M_\ell$. The projection error is averaged over the four charts and is always evaluated on the complete local table.}
\label{tab:appendix-anchor-budget}
\begin{tabular}{lccccc}
\hline
Model & Anchor frac. & Anchors & Rank / $M_\ell$ & Mean $\varepsilon_{\rm proj}$ & All charts accepted \\
\hline
GPT-2
& $0.50\%$ & $50$  & $32/32$ & $6.31\times10^{-7}$ & Yes \\
& $0.64\%$ & $64$  & $32/32$ & $6.19\times10^{-7}$ & Yes \\
& $0.75\%$ & $75$  & $32/32$ & $6.74\times10^{-7}$ & Yes \\
& $1.00\%$ & $100$ & $32/32$ & $5.88\times10^{-7}$ & Yes \\
\hline
Llama-2-7B
& $0.50\%$ & $50$  & $50/96$ & $2.3794\times10^{-1}$ & No \\
& $0.64\%$ & $64$  & $64/96$ & $1.9669\times10^{-1}$ & No \\
& $0.75\%$ & $75$  & $75/96$ & $1.5860\times10^{-1}$ & No \\
& $1.00\%$ & $100$ & $96/96$ & $5.40\times10^{-7}$ & Yes \\
& $1.20\%$ & $120$ & $96/96$ & $4.57\times10^{-7}$ & Yes \\
\hline
Mistral-7B
& $0.50\%$ & $50$  & $50/96$ & $3.8141\times10^{-2}$ & No \\
& $0.64\%$ & $64$  & $64/96$ & $3.1573\times10^{-2}$ & No \\
& $0.75\%$ & $75$  & $75/96$ & $2.5444\times10^{-2}$ & No \\
& $1.00\%$ & $100$ & $96/96$ & $8.40\times10^{-8}$ & Yes \\
& $1.20\%$ & $120$ & $96/96$ & $7.25\times10^{-8}$ & Yes \\
\hline
Mistral-7B-Instruct
& $0.50\%$ & $50$  & $50/96$ & $3.6489\times10^{-2}$ & No \\
& $0.64\%$ & $64$  & $64/96$ & $3.0279\times10^{-2}$ & No \\
& $0.75\%$ & $75$  & $75/96$ & $2.4360\times10^{-2}$ & No \\
& $1.00\%$ & $100$ & $96/96$ & $8.36\times10^{-8}$ & Yes \\
& $1.20\%$ & $120$ & $96/96$ & $7.02\times10^{-8}$ & Yes \\
\hline
\end{tabular}
\end{table*}

The transition is exceptionally sharp. For every $96$-dimensional model, the scaffold rank tracks the anchor count exactly at the subcritical budgets,
\[
50,\quad64,\quad75,
\]
and all such scaffolds fail the full-table acceptance test. At the first tested budget above the chart dimension, $100$ anchors, every chart in every model attains rank $96$ and the projection error collapses to the $10^{-7}$--$10^{-8}$ scale. Increasing the budget to $120$ anchors preserves this regime.

GPT-2 provides the complementary case. Its charts have dimension only $32$, so the smallest tested budget, $50$ anchors or $0.5\%$ of the table, is already above the dimensional threshold. Every tested GPT-2 scaffold is therefore full rank and accepted, with projection error remaining at the $10^{-7}$ scale throughout the entire sweep.

Two conclusions follow directly from this diagnostic. First, the amount of data required to recover an accepted scaffold is governed primarily by the \emph{local chart dimension}, rather than by the size of the embedding table. Second, the production budgets used in the headline experiments are conservative: GPT-2 already succeeds at $0.5\%$ although the main experiments use $0.75\%$, and all three $96$-dimensional models already succeed at $1.0\%$ although the main experiments use $1.2\%$. Thus the reported compression results do not rely on tuning the anchor fraction to a fragile minimum-success point.

\subsection{Experiment 7: three-seed robustness of the headline configuration}
\label{app:exp7}

The complete rate--distortion curves are expensive because every point requires constructing and evaluating the shared codec. We therefore use a fixed seed for the full curves and separately test the final headline operating point under three independent seeds,
\[
13,\qquad29,\qquad47.
\]
For each seed the full headline pipeline is rerun, including embedding-table sampling, representation fitting, chart construction, scaffold discovery, and shared-codeword construction. The codec dimensions and transmitted rate are held fixed.

Table~\ref{tab:appendix-seeds} reports mean and standard deviation over the three runs. The final end-to-end relative error is highly stable: the standard deviation is approximately $3.1\times10^{-4}$ for GPT-2 and between $1.0\times10^{-3}$ and $1.6\times10^{-3}$ for the $7$B tables. Relative to the corresponding means, all four end-to-end standard deviations are below $0.2\%$.

\begin{table*}[t]
\centering
\scriptsize
\setlength{\tabcolsep}{3pt}
\caption{Experiment~7: robustness of the final $T=7$ headline operating point over seeds $\{13,29,47\}$. Values are mean $\pm$ standard deviation.}
\label{tab:appendix-seeds}
\begin{tabular}{lccccc}
\hline
Model & PCA-space rel. error & End-to-end rel. error & Mean cosine & Recall@10 & Bits/token \\
\hline
GPT-2
& $0.26113\pm0.00063$
& $0.74506\pm0.00031$
& $0.66730\pm0.00084$
& $0.29727\pm0.01412$
& $324.1472\pm0$ \\
Llama-2-7B
& $0.60226\pm0.01186$
& $0.91820\pm0.00158$
& $0.39200\pm0.00323$
& $0.33103\pm0.00466$
& $2739.968\pm0$ \\
Mistral-7B
& $0.57864\pm0.00966$
& $0.91598\pm0.00120$
& $0.39796\pm0.00303$
& $0.36715\pm0.00710$
& $2739.968\pm0$ \\
Mistral-7B-Instruct
& $0.62136\pm0.00638$
& $0.92099\pm0.00105$
& $0.37600\pm0.00228$
& $0.32091\pm0.00454$
& $2739.968\pm0$ \\
\hline
\end{tabular}
\end{table*}

The rate is invariant across seeds, as expected from the fixed codec dimensions. More importantly, the small end-to-end variance shows that the headline distortion values are not favorable stochastic outliers. The reduced-space error varies somewhat more strongly, particularly on Llama-2-7B and Mistral-7B, but this variability remains small after evaluation in the original embedding space. Mean cosine similarity is similarly stable. Nearest-neighbor recall@10 has the largest relative variation, which is consistent with its discrete ranking nature, but it does not induce corresponding instability in either reconstruction error or cosine fidelity.

Taken together with the mechanism ablations, the seed study strengthens the interpretation of the main results. The observed behavior is reproduced across independently reconstructed orbit geometries: the data-aligned scaffold consistently reduces projection error relative to the rank-matched Gaussian control; gluing-aware refinement reduces the global objective at fixed rate; redundant synthesis produces the predicted cancellation mechanism; the orbit construction can be compiled away exactly; accepted scaffolds are recovered from approximately one percent of the table once the local dimensional threshold is reached; and the final compressed operating point remains stable across independent runs.

\section{Proofs of the main results}
\label{app:proofs}

\textbf{{Proof of Theorem \ref{thm:redundant orbit add comprssion decomposition} (Redundant orbit chart error decomposition).}}
We use throughout the pointwise table-level extensions of the operators. If
\(T:\mathcal H\to K\) is bounded and \(T_D: \mathcal H^D\to K^D\) is defined by
\(T_D(e_1,\ldots,e_D)=(Te_1,\ldots,Te_D)\), then, with respect to the weighted empirical norm,
\[
\|T_DE\|_\mu^2
=
\sum_{n=1}^D\mu_n\|Te_n\|^2
\leq
\|T\|^2\sum_{n=1}^D\mu_n\|e_n\|^2
=
\|T\|^2\|E\|_\mu^2.
\]
Thus \(\|T_D\|\leq\|T\|\), and equality follows by testing on tables with a single nonzero entry. Hence \(\|T_D\|=\|T\|\). In particular, \(\|\Gamma^D\|=\|\Gamma\|\). If \(R^*R=A\,\mathrm{Id}_{\mathcal H}\), then also \(R_D^*R_D=A\,\mathrm{Id}_{\mathcal H^D}\), so \(\|R_D\|=\sqrt A\) and \(\|\Gamma_A^D\|=\|\Gamma_A\|=1/\sqrt A\).

Let \(\eta := R_D E - \widehat{Z} \in \mathcal{Z}^D\), where \(\widehat{Z} = (\widehat{Z}^{(1)}, \dots, \widehat{Z}^{(L)})\). By definition we have \(\widehat{E} = \Gamma^D \widehat{Z}\). Therefore
\begin{equation}
\label{eq:E-hatE}
E - \widehat{E} = E - \Gamma^D \widehat{Z} = \bigl(E - \Gamma^D R_D E\bigr) + \Gamma^D \eta.
\end{equation}
Taking the weighted empirical norm and applying the triangle inequality yields
\[
\|E - \widehat{E}\|_\mu \le \|E - \Gamma^D R_D E\|_\mu + \|\Gamma^D \eta\|_\mu = \varepsilon_{\mathrm{glue}}(E) + \|\Gamma^D \eta\|_\mu.
\]
Since \(\ker \Gamma^D\) is a closed subspace of \(\mathcal{Z}^D\), we have the orthogonal decomposition
\[
\eta = \Pi_{\ker \Gamma^D} \eta + \bigl(I - \Pi_{\ker \Gamma^D}\bigr) \eta.
\]
Applying \(\Gamma^D\) and using that \(\Gamma^D\) annihilates \(\ker \Gamma^D\) gives \(\Gamma^D \eta = \Gamma^D \bigl(I - \Pi_{\ker \Gamma^D}\bigr) \eta\). Consequently
\[
\|\Gamma^D \eta\|_\mu \le \|\Gamma\| \cdot \bigl\| \bigl(I - \Pi_{\ker \Gamma^D}\bigr) \eta \bigr\|_\mu,
\]
where we used the operator-norm bound \(\|\Gamma^D\| \le \|\Gamma\|\). Substituting back produces
\[
\|E - \widehat{E}\|_\mu \le \varepsilon_{\mathrm{glue}}(E) + \|\Gamma\| \cdot \bigl\| \bigl(I - \Pi_{\ker \Gamma^D}\bigr) \eta \bigr\|_\mu.
\]
Orthogonality of the decomposition of \(\eta\) immediately implies
\[
\bigl\| \bigl(I - \Pi_{\ker \Gamma^D}\bigr) \eta \bigr\|_\mu^2 = \|\eta\|_\mu^2 - \|\Pi_{\ker \Gamma^D} \eta\|_\mu^2,
\]
and therefore
\[
\|E - \widehat{E}\|_\mu \le \varepsilon_{\mathrm{glue}}(E) + \|\Gamma\| \bigl( \|\eta\|_\mu^2 - \|\Pi_{\ker \Gamma^D} \eta\|_\mu^2 \bigr)^{1/2}.
\]
By definition we have \(\varepsilon_{\ker}(E) := \|\Pi_{\ker \Gamma^D} \eta\|\), which is exactly the claimed general bound.

Now assume that $\{R_\ell\}_{\ell=1}^L$ is $A$-tight and $\Gamma=\Gamma_A=A^{-1}R^*$. Then $\Gamma_A R=\mathrm{Id}_{\mathcal H}$, hence $\varepsilon_{\rm glue}(E)=0$, and $\ker \Gamma_A^D=(\operatorname{Ran}R_D)^\perp$. Therefore by (\ref{eq:E-hatE})
\begin{equation}
E-\widehat E
=
\Gamma_A^D\eta
=
\Gamma_A^D\Pi_{\operatorname{Ran}R_D}\eta.
\end{equation}

For $y\in\operatorname{Ran}R_D$, write $y=R_DF$. Since $R_D^*R_D=A\,\mathrm{Id}_{\mathcal H^D}$, we have $\|y\|_\mu^2=A\|F\|_\mu^2$ and $\Gamma_A^D y=F$. Hence
\begin{equation}
\|\Gamma_A^D y\|_\mu
=
\frac{1}{\sqrt A}\|y\|_\mu,
\qquad
y\in\operatorname{Ran}R_D.
\end{equation}
Applying this identity to $\Pi_{\operatorname{Ran}R_D}\eta$ gives
\begin{equation}
\|E-\widehat E\|_\mu
=
\frac{1}{\sqrt A}
\|\Pi_{\operatorname{Ran}R_D}\eta\|_\mu.
\end{equation}

Since $R_DE\in\operatorname{Ran}R_D$,
\begin{equation}
(I-\Pi_{\operatorname{Ran}R_D})\eta
=
-(I-\Pi_{\operatorname{Ran}R_D})\widehat Z.
\end{equation}
Therefore
\begin{equation}
\|\Pi_{\operatorname{Ran}R_D}\eta\|_\mu^2
=
\|\eta\|_\mu^2
-
\operatorname{dist}\big(\widehat Z,\operatorname{Ran}R_D\big)^2
=
\varepsilon_{\rm loc}^2(E)-\varepsilon_{\rm inc}^2(E).
\end{equation}
Substitution yields
\begin{equation}
\|E-\widehat E\|_\mu
=
\frac{1}{\sqrt A}
\big(
\varepsilon_{\rm loc}^2(E)-\varepsilon_{\rm inc}^2(E)
\big)^{1/2},
\end{equation}
which proves the canonical $A$-tight statement.
\qed

\textbf{Proof of Proposition~\ref{prop:nested sequential} (Nested sequential shared-codebook residual identity).}
Fix a sequential step $s$ and write $\ell:=\ell_s$, $G:=G^{(s-1)}$, $U:=U^{(s)}$, and $V:=V^{(s)}$. Since $R_\ell R_\ell^*=I_{\mathcal Z_\ell}$, the operator $P_\ell:=R_\ell^*R_\ell$ is the orthogonal projector onto $\operatorname{Ran}R_\ell^*$. Moreover, $R_\ell^*$ is an isometry from $\mathcal Z_\ell$ into $\mathcal H$. Using \eqref{eq:nested-residual-update},
\begin{equation}
G^{(s)}
=
(I-P_\ell)G
+
R_\ell^*
\left(
R_\ell G-\frac{1}{A}V
\right).
\end{equation}
The two terms on the right are orthogonal. Applying this identity tokenwise and summing with the weights $\mu_n$ gives
\begin{equation}
\|G^{(s)}\|_\mu^2
=
\|(I-P_\ell)G\|_\mu^2
+
\left\|
R_\ell G-\frac{1}{A}V
\right\|_\mu^2.
\end{equation}
Similarly,
\begin{equation}
\|G\|_\mu^2
=
\|(I-P_\ell)G\|_\mu^2
+
\|R_\ell G\|_\mu^2.
\end{equation}
Subtracting the last two identities and using $U=A R_\ell G$ yields
\begin{equation}
\|G^{(s)}\|_\mu^2
=
\|G\|_\mu^2
-
\frac{1}{A^2}\|U\|_\mu^2
+
\frac{1}{A^2}\|U-V\|_\mu^2,
\end{equation}
which is \eqref{eq:nested-residual-identity}. The stated nonincrease condition follows immediately.
\qed

\textbf{Proof of Proposition \ref{prop:gluing aware refinement} (Scaffold-aligned post-gluing refinement).}
The argument is the post-gluing analogue of the residual decomposition used in Proposition~2.5 for the sequential initialization. There, a new stage is fitted to the residual left by the stages already constructed. Here, all \(LT\) stages are already present, so we first remove the particular stage containing \(c_{\mathrm{old}}\) and ask what that stage should represent in the completed global reconstruction.

Let
\[
V_{\ell,t}^{\mathrm{old}}
:=
\bigl(
c_{\ell,t,i_{\ell,1,t}},
\ldots,
c_{\ell,t,i_{\ell,D,t}}
\bigr)
\in \mathcal Z_\ell^D
\]
denote the current compiled table contributed by stage \((\ell,t)\). Since canonical \(A\)-tight gluing gives
\[
\widehat E_{\mathrm{old}}
=
\frac{1}{A}
\sum_{j=1}^L
R_j^*\widehat Z^{(j)},
\]
the contribution of this particular stage to the global reconstruction is
$$
\frac{1}{A}R_\ell^*V_{\ell,t}^{\mathrm{old}}.
$$
Remove it and define
\[
\widehat E_{-\ell,t}
:=
\widehat E_{\mathrm{old}}
-
\frac{1}{A}R_\ell^*V_{\ell,t}^{\mathrm{old}},
\qquad
H_{\ell,t}
:=
E-\widehat E_{-\ell,t}.
\]
Thus \(H_{\ell,t}\) is the global residual that remains after every decoder contribution except stage \((\ell,t)\) has been taken into account.

Exactly as in the sequential construction, the chart-visible target associated with this residual is
\[
U_{\ell,t}^{\star}
:=
A R_\ell H_{\ell,t}
\in \mathcal Z_\ell^D.
\]
This is the post-gluing counterpart of the sequential target
$
U^{(s)}
=
A R_{\ell_s}G^{(s-1)}
$
from Proposition~\ref{prop:nested sequential}. The difference is that \(U^{(s)}\) is formed before a new stage is inserted, whereas \(U_{\ell,t}^{\star}\) is formed after removing an already existing stage from the completed reconstruction.

We first show that \(U_{\ell,t}^{\star}\) plays exactly the same role in refinement that \(U^{(s)}\) plays during initialization. Set
$
P_\ell:=R_\ell^*R_\ell.
$
Since
$
R_\ell R_\ell^*=I_{\mathcal Z_\ell},
$
the operator \(P_\ell\) is the orthogonal projector onto \(\operatorname{Ran}R_\ell^*\), and \(R_\ell^*:\mathcal Z_\ell\to \mathcal H\) is an isometry. For any candidate replacement stage table
$
V\in \mathcal Z_\ell^D,
$
the corresponding global reconstruction is
\[
\widehat E(V)
:=
\widehat E_{-\ell,t}
+
\frac{1}{A}R_\ell^*V.
\]
Hence
\[
E-\widehat E(V)
=
H_{\ell,t}
-
\frac{1}{A}R_\ell^*V.
\]
Using
\[
U_{\ell,t}^{\star}
=
A R_\ell H_{\ell,t},
\]
we obtain
\begin{align*}
H_{\ell,t}
-
\frac{1}{A}R_\ell^*V
&=
(I-P_\ell)H_{\ell,t}
+
P_\ell H_{\ell,t}
-
\frac{1}{A}R_\ell^*V\\
&=
(I-P_\ell)H_{\ell,t}
+
R_\ell^*
\left(
R_\ell H_{\ell,t}
-
\frac{1}{A}V
\right)\\
&=
(I-P_\ell)H_{\ell,t}
+
\frac{1}{A}R_\ell^*
\left(
U_{\ell,t}^{\star}-V
\right).
\end{align*}
The two terms on the right are orthogonal: the first belongs to
\[
\ker R_\ell
=
\bigl(\operatorname{Ran}R_\ell^*\bigr)^\perp,
\]
whereas the second belongs to \(\operatorname{Ran}R_\ell^*\). Applying the identity tokenwise, summing with the weights \(\mu_n\), and using that \(R_\ell^*\) is an isometry gives
\begin{equation}
\label{eq:post-gluing-stage-identity}
\left\|
E-\widehat E(V)
\right\|_\mu^2
=
\left\|
(I-P_\ell)H_{\ell,t}
\right\|_\mu^2
+
\frac{1}{A^2}
\left\|
U_{\ell,t}^{\star}-V
\right\|_\mu^2.
\end{equation}
This is the refinement analogue of the identity in Proposition~\ref{prop:nested sequential}. The first term is invisible to chart \(\ell\) and is independent of the replacement stage. Therefore every possible improvement produced by stage \((\ell,t)\) is completely determined by how well its realized table approximates \(U_{\ell,t}^{\star}\).

We now specialize this identity to the single codeword \(c_{\mathrm{old}}\). Recall that
$c_{\rm old}:=c_{\ell,t,r_0}\in \mathcal Z_\ell$ and
$$
I=\{n:i_{\ell,n,t}=r_0\}
\text{ with }
m_I:=\sum_{n\in I}\mu_n>0.
$$

The assignments are held fixed, as required in the statement of the proposition. For any candidate codeword \(c\in \mathcal Z_\ell\), define the corresponding stage table \(V(c)\in \mathcal Z_\ell^D\) by
\[
V(c)_n
:=
\begin{cases}
c, & n\in I,\\
(V_{\ell,t}^{\mathrm{old}})_n, & n\notin I.
\end{cases}
\]
Then
\[
V(c_{\mathrm{old}})
=
V_{\ell,t}^{\mathrm{old}}.
\]
Write
\[
U_{\ell,t}^{\star}
=
\bigl(
u_{\ell,t,1}^{\star},
\ldots,
u_{\ell,t,D}^{\star}
\bigr).
\]
Since only the rows indexed by \(I\) depend on \(c\),
\begin{align}
\left\|
U_{\ell,t}^{\star}-V(c)
\right\|_\mu^2
&=
\sum_{n\notin I}
\mu_n
\left\|
u_{\ell,t,n}^{\star}
-
(V_{\ell,t}^{\mathrm{old}})_n
\right\|^2
+
\sum_{n\in I}
\mu_n
\left\|
u_{\ell,t,n}^{\star}-c
\right\|^2.
\label{eq:refinement-codeword-splitting}
\end{align}
The first term is independent of \(c\). Hence the unique unconstrained minimizer \(c^\star\) from the statement is precisely the minimizer of
\[
c
\longmapsto
\sum_{n\in I}
\mu_n
\left\|
u_{\ell,t,n}^{\star}-c
\right\|^2.
\]
Because \(m_I>0\), this function is strictly convex, and its unique minimizer is the weighted centroid
\begin{equation}
\label{eq:explicit-refinement-target}
c^\star
=
\frac{1}{m_I}
\sum_{n\in I}
\mu_n
u_{\ell,t,n}^{\star}.
\end{equation}
Since
\[
u_{\ell,t,n}^{\star}
=
A R_\ell
\bigl(
e_n-\widehat e_{-\ell,t,n}
\bigr),
\]
where \(\widehat e_{-\ell,t,n}\) denotes the \(n\)-th row of \(\widehat E_{-\ell,t}\), we obtain the explicit formula
\begin{equation}
\label{eq:explicit-refinement-target-global}
c^\star
=
\frac{A}{m_I}
\sum_{n\in I}
\mu_n
R_\ell
\bigl(
e_n-\widehat e_{-\ell,t,n}
\bigr).
\end{equation}
Thus \(c^\star\) is not an abstract auxiliary point: it is the weighted centroid of the exact post-gluing chart target seen by the tokens currently assigned to \(c_{\mathrm{old}}\).

The weighted centroid identity now gives, for every \(c\in Z_\ell\),
\begin{equation}
\label{eq:centroid-pythagorean-refinement}
\sum_{n\in I}
\mu_n
\left\|
u_{\ell,t,n}^{\star}-c
\right\|^2
=
\sum_{n\in I}
\mu_n
\left\|
u_{\ell,t,n}^{\star}-c^\star
\right\|^2
+
m_I
\left\|
c-c^\star
\right\|^2.
\end{equation}
Indeed,
\[
u_{\ell,t,n}^{\star}-c
=
\bigl(
u_{\ell,t,n}^{\star}-c^\star
\bigr)
+
\bigl(
c^\star-c
\bigr),
\]

and the cross term vanishes because, by \eqref{eq:explicit-refinement-target},
$
\sum_{n\in I}\mu_n\bigl(u^\star_{\ell,t,n}-c^\star\bigr)=0.
$

Combining
\eqref{eq:post-gluing-stage-identity},
\eqref{eq:refinement-codeword-splitting},
and
\eqref{eq:centroid-pythagorean-refinement},
all terms independent of \(c\) can be collected into a constant \(C\), yielding
\begin{equation}
\label{eq:global-refinement-distance-identity}
\left\|
E-\widehat E(V(c))
\right\|_\mu^2
=
C
+
\frac{m_I}{A^2}
\left\|
c-c^\star
\right\|^2.
\end{equation}
Therefore the effect of replacing \(c_{\mathrm{old}}\) by any candidate \(c\) is measured exactly by its distance from the post-gluing optimum \(c^\star\). In particular, since
\[
\widehat E_{\mathrm{old}}
=
\widehat E(V(c_{\mathrm{old}}))
\]
and
\[
\widehat E_{\mathrm{new}}
=
\widehat E(V(c_{\mathrm{new}})),
\]
subtracting the two instances of
\eqref{eq:global-refinement-distance-identity}
gives the exact relation
\begin{equation}
\label{eq:exact-old-new-refinement}
\left\|
E-\widehat E_{\mathrm{new}}
\right\|_\mu^2
-
\left\|
E-\widehat E_{\mathrm{old}}
\right\|_\mu^2
=
\frac{m_I}{A^2}
\left(
\left\|
c_{\mathrm{new}}-c^\star
\right\|^2
-
\left\|
c_{\mathrm{old}}-c^\star
\right\|^2
\right).
\end{equation}

It remains to control how far the actual OrBIT update can lie from the unconstrained target \(c^\star\). The method does not insert \(c^\star\) directly into the decoder. Instead, the target is restricted to the learned scaffold, realized through the orbit dictionary, and compiled through the finite-bit decoder. Recall the definitions in the statement, 

\[
\delta_{\mathrm{scaf}}
=
\left\|
c^\star-\Pi_{S_\ell}c^\star
\right\|,
\qquad
\delta_{\mathrm{orb}}
=
\left\|
\Pi_{S_\ell}c^\star-\widetilde c
\right\|,
\qquad
\delta_{\mathrm{comp}}
=
\left\|
\widetilde c-c_{\mathrm{new}}
\right\|.
\]
Now write
\[
c^\star - \widetilde c
=
\underbrace{c^\star - \Pi_{S_\ell}c^\star}_{\in S_\ell^\perp}
+
\underbrace{\Pi_{S_\ell}c^\star - \widetilde c}_{\in S_\ell},
\]
where
$
\widetilde c
=
\operatorname{CCDS}_{D_\ell}(c^\star)
\in S_\ell.
$
Those two terms are exactly orthogonal. Hence
\[
\|c^\star - \widetilde c\|^2
=
\delta_{\rm scaf}^2
+
\delta_{\rm orb}^2.
\]

Since
$
c_{\rm new} - c^\star
=
(c_{\rm new} - \widetilde c)
+
(\widetilde c - c^\star).
$
By triangle inequality
\[
\|c_{\rm new} - c^\star\|
\le
\Delta
:=
\sqrt{\delta_{\rm scaf}^2 + \delta_{\rm orb}^2}
+
\delta_{\rm comp}.
\]

Substituting this estimate into
\eqref{eq:exact-old-new-refinement}
yields
\begin{align*}
\left\|
E-\widehat E_{\mathrm{new}}
\right\|_\mu^2
-
\left\|
E-\widehat E_{\mathrm{old}}
\right\|_\mu^2
&\leq
\frac{m_I}{A^2}
\left(
\Delta^2
-
\left\|
c_{\mathrm{old}}-c^\star
\right\|^2
\right).
\end{align*}
Equivalently,
\[
\left\|
E-\widehat E_{\mathrm{new}}
\right\|_\mu^2
\leq
\left\|
E-\widehat E_{\mathrm{old}}
\right\|_\mu^2
-
\frac{m_I}{A^2}
\left(
\left\|
c_{\mathrm{old}}-c^\star
\right\|^2
-
\Delta^2
\right),
\]
which is the claimed quantitative bound.

Finally, if
\[
\Delta
<
\left\|
c_{\mathrm{old}}-c^\star
\right\|,
\]
then
\[
\left\|
c_{\mathrm{old}}-c^\star
\right\|^2
-
\Delta^2
>
0.
\]
Since \(m_I>0\) and \(A>0\), the amount subtracted above is strictly positive. Therefore
\[
\left\|
E-\widehat E_{\mathrm{new}}
\right\|_\mu^2
<
\left\|
E-\widehat E_{\mathrm{old}}
\right\|_\mu^2.
\]
\qed

\textbf{Proof of Lemma \ref{lem:refinement-scaffold-error} (Refinement scaffold error as unexplained global residual).}
Recall we assumed the unit-gain normalization
$
R_\ell R_\ell^*=I_{\mathcal Z_\ell},
$
 canonical \(A\)-tight gluing, and we fixed a chart \(\ell\), stage \(t\), and codeword \(c_{\mathrm{old}}:=c_{\ell,t,r_0}\in \mathcal Z_\ell\) such that 
$$
I=\{n:i_{\ell,n,t}=r_0\}
\text{ with }
m_I:=\sum_{n\in I}\mu_n>0.
$$
For \(n\in I\), let \(\widehat e_{-\ell,t,n}\) denote the current global reconstruction with the contribution of \(c_{\mathrm{old}}\) removed. Under canonical \(A\)-tight gluing,
\[
\widehat e_{\mathrm{old},n}
=
\widehat e_{-\ell,t,n}
+
\frac{1}{A}R_\ell^*c_{\mathrm{old}}.
\]
Therefore
\[
e_n-\widehat e_{-\ell,t,n}
=
e_n-\widehat e_{\mathrm{old},n}
+
\frac{1}{A}R_\ell^*c_{\mathrm{old}}.
\]

Recall \(c^\star\in \mathcal Z_\ell\) denote the unique unconstrained minimizer of the conditional post-gluing distortion obtained by varying only \(c_{\mathrm{old}}\), with all remaining codewords and assignments fixed. By Equation (\ref{eq:explicit-refinement-target-global} 
$$
c^\star
=
\frac{A}{m_I}
\sum_{n\in I}
\mu_n
R_\ell
\bigl(
e_n-\widehat e_{-\ell,t,n}
\bigr).
$$

Substituting the identity for
\(e_n-\widehat e_{-\ell,t,n}\) yields
\begin{align*}
c^\star
&=
\frac{A}{m_I}
\sum_{n\in I}
\mu_n
R_\ell
\left(
e_n-\widehat e_{\mathrm{old},n}
+
\frac{1}{A}R_\ell^*c_{\mathrm{old}}
\right) \\
&=
\frac{A}{m_I}
\sum_{n\in I}
\mu_n
R_\ell
\bigl(
e_n-\widehat e_{\mathrm{old},n}
\bigr)
+
\frac{1}{m_I}
\sum_{n\in I}
\mu_n
R_\ell R_\ell^*c_{\mathrm{old}} \\
&=
\frac{A}{m_I}
\sum_{n\in I}
\mu_n
R_\ell
\bigl(
e_n-\widehat e_{\mathrm{old},n}
\bigr)
+
c_{\mathrm{old}}.
\end{align*}
Hence
\begin{equation}
\label{eq:c start eq}
c^\star
=
c_{\mathrm{old}}+
\frac{A}{m_I}
\sum_{n\in I}
\mu_n
R_\ell
\bigl(
e_n-\widehat e_{\mathrm{old},n}
\bigr).
\end{equation}
Applying \(I-\Pi_{S_\ell}\) to (\ref{eq:c start eq}) and taking norms  gives the first claim.

Assume now that \(c_{\mathrm{old}}\in S_\ell\), then
$
(I-\Pi_{S_\ell})c_{\mathrm{old}}=0.
$
Applying \(I-\Pi_{S_\ell}\) to (\ref{eq:c start eq}) and taking norms gives the second claim.
\qed

\textbf{Proof of Corollary \ref{cor:refinement-scaffold-compatibility} (Collapse under cross-chart scaffold compatibility).}
Canonical \(A\)-tight gluing gives
\[
\widehat e_{\mathrm{old},n}
=
\frac{1}{A}
\sum_{j=1}^L
R_j^*\widehat z_{j, n}.
\]
Applying \(R_\ell\) yields
\[
R_\ell\widehat e_{\mathrm{old},n}
=
\frac{1}{A}
\sum_{j=1}^L
R_\ell R_j^*\widehat z_{j, n}.
\]
Since $\widehat z_{j, n}\in S_j$, by the compatibility assumption,
\[
R_\ell R_j^*\widehat z_{j, n}
\in
R_\ell R_j^*S_j
\subseteq
S_\ell
\]
for every \(j\). Since \(S_\ell\) is a linear subspace,
$
R_\ell\widehat e_{\mathrm{old},n}
\in
S_\ell.
$
Therefore
\[
(I-\Pi_{S_\ell})
R_\ell\widehat e_{\mathrm{old},n}
=
0.
\]

Using Lemma~\ref{lem:refinement-scaffold-error},
\begin{align*}
\delta_{\mathrm{scaf}}
&=
\frac{A}{m_I}
\left\|
(I-\Pi_{S_\ell})
\sum_{n\in I}
\mu_n
R_\ell
\bigl(
e_n-\widehat e_{\mathrm{old},n}
\bigr)
\right\| \\
&=
\frac{A}{m_I}
\left\|
\sum_{n\in I}
\mu_n
(I-\Pi_{S_\ell})
R_\ell e_n
\right\|.
\end{align*}


Set 
\[
a_n
:=
(I-\Pi_{S_\ell})R_\ell e_n,
\qquad
\overline a_I
:=
\frac{1}{m_I}
\sum_{n\in I}\mu_n a_n.
\]
Hence
\begin{equation}
\label{eq:delta scaff}
\delta_{\rm scaf}
=
A\|\overline a_I\|.
\end{equation}

By weighted Cauchy--Schwarz,
\[
\left\|
\sum_{n\in I}\mu_n a_n
\right\|^2
\le
\left(\sum_{n\in I}\mu_n\right)
\left(\sum_{n\in I}\mu_n\|a_n\|^2\right)
=
m_I\sum_{n\in I}\mu_n\|a_n\|^2.
\]
Since
$
\overline a_I
=
\frac{1}{m_I}\sum_{n\in I}\mu_n a_n,
$
it follows that
\[
m_I\|\overline a_I\|^2
\le
\sum_{n\in I}\mu_n\|a_n\|^2.
\]
Now fix a stage \((\ell,t)\). For every nonempty assignment cell 
\[
I_r:=\{n:i_{\ell,n,t}=r\},
\qquad
m_{I_r}:=\sum_{n\in I_r}\mu_n>0,
\]
assume the same compatibility and scaffold-membership assumptions hold for every cell, hence by (\ref{eq:delta scaff})
\[
\frac{m_{I_r}}{A^2}\delta_{{\rm scaf},r}^2
=
m_{I_r}\|\overline a_{I_r}\|^2
\le
\sum_{n\in I_r}\mu_n\|a_n\|^2.
\]
The nonempty cells \(I_r\) form a partition of
\(\{1,\ldots,D\}\), so summing over \(r\) and taking into account the definition of the chart projection error (\ref{eq:scaffold-accept}) yields
\[
\sum_{r:\,m_{I_r}>0}
\frac{m_{I_r}}{A^2}\delta_{{\rm scaf},r}^2
\le
\sum_{n=1}^D
\mu_n\|(I-\Pi_{S_\ell})R_\ell e_n\|^2
=
\varepsilon_{{\rm proj},\ell}^2.
\]
\qed

\end{document}